\documentclass[twoside,11pt]{article}
\usepackage{color}
\usepackage[usenames,dvipsnames]{xcolor}

\usepackage{jmlr2e}

\jmlrheading{1}{2012}{1-30}{4/00}{10/00}{Theja Tulabandhula and Cynthia Rudin}
\ShortHeadings{Machine Learning with Operational Costs}{Tulabandhula et al.}
\firstpageno{1}
\title{Machine Learning with Operational Costs}
\author{\name Theja Tulabandhula \email theja@mit.edu \\
       \addr Department of Electrical Engineering and Computer Science\\
       Massachusetts Institute of Technology\\
       Cambridge, MA 02139, USA
       \AND
       \name Cynthia Rudin \email rudin@mit.edu \\
       \addr MIT Sloan School of Management and Operations Research Center\\
       Massachusetts Institute of Technology\\
       Cambridge, MA 02139, USA
       }
\editor{}

\usepackage{natbib}
\usepackage{fullpage}
\usepackage{wrapfig}
\usepackage{amsmath}
\usepackage{subfigure}
\usepackage{tikz}

\graphicspath{{./figures/}{./}}

\usepackage{mathtools}
\DeclarePairedDelimiter{\ceil}{\lceil}{\rceil}
\newcommand{\qed}{\nobreak \ifvmode \relax \else
      \ifdim\lastskip<1.5em \hskip-\lastskip
      \hskip1.5em plus0em minus0.5em \fi \nobreak
      \vrule height0.75em width0.5em depth0.25em\fi}

\def\matXbar{\tilde{X}}
\def\sign{\textrm{sign}}
\def\gammaS{\bar{\gamma}}
\def\betaS{\tilde{\beta}}
\def\hS{\tilde{h}}
 \def\xt{\tilde{x}}
 \def\scalej{\frac{n^{1/r}X_{b}B_{b}}{\|h_{j}\|_{r}}}
  \def\invscalej{\frac{\|h_{j}\|_{r}}{n^{1/r}X_{b}B_{b}}}

\def\R{\mathbb{R}}
\def\X{\mathcal{X}}
\def\F{\mathcal{F}}
\def\Y{\mathcal{Y}}
\def\Fr{\mathcal{F}^{R}}
\def\Obj{\textrm{\rm OpCost}}
\def\argminpi{\hbox{ \raise-1.6mm\hbox{$\textstyle
    \mathrm{argmin} \atop \pi\in\Pi$}}}
\def\argminf{\hbox{ \raise-1.6mm\hbox{$\textstyle
    \mathrm{argmin} \atop f\in\F^{unc}$}}}
\def\minpi{\hbox{ \raise-1.6mm\hbox{$\textstyle
    \mathrm{min} \atop \pi\in\Pi$}}}    
\def\argmin{\textrm{argmin}}
\def\cuppi{\hbox{ \raise-1.6mm\hbox{$\textstyle
    \cap \atop g\in \F_{good}$}}}
\def\cuuppi{\hbox{ \raise-1.6mm\hbox{$\textstyle
    \cup \atop g\in \F_{good}$}}}

\begin{document}
\maketitle

\begin{abstract}%
\textcolor{Black}{ %COLOR FIRST LAYER
This work proposes a way to align statistical modeling with decision making.
We provide a method that propagates the uncertainty in predictive modeling to
the uncertainty in operational cost, where operational cost is the amount spent
by the practitioner in solving the problem. The method allows us to explore the
range of operational costs associated with the set of reasonable statistical models,
so as to provide a useful way for practitioners to understand uncertainty. To do
this, the operational cost is cast as a regularization term in a learning algorithm's
objective function, allowing either an optimistic or pessimistic view of possible
costs, depending on the regularization parameter. From another perspective, if
we have prior knowledge about the operational cost, for instance that it should be
low, this knowledge can help to restrict the hypothesis space, and can help with
generalization. We provide a theoretical generalization bound for this scenario.
We also show that learning with operational costs is related to robust optimization.
} %COLOR

\end{abstract}

%%%%%%%%%%%%%%%%%%%Section: Introduction
\section{Introduction} \label{sec:intro}
\textcolor{Black}{ %COLOR FIRST LAYER
Machine learning algorithms are used to produce predictions, and these predictions are often used to make a policy or plan of action afterwards, where there is a cost to implement the policy. 
In this work, we would like to understand how the uncertainty in predictive modeling can translate into the uncertainty in the cost for implementing the policy. This would help us answer questions like: 
\textcolor{Black}{ %COLOR SECOND LAYER
\begin{enumerate}
\item[Q1.] ``What is a reasonable amount to allocate for this task so we can react best to whatever nature brings?"
\item[Q2.] ``Can we produce a reasonable probabilistic model, supported by data, where we might expect to pay a specific amount?"
\item[Q3.] ``Can our intuition about how much it will cost to solve a problem help us produce a better probabilistic model?" 
\end{enumerate}
}
%``given that we will spend a certain amount, what is the most likely scenario, supported by data, that would have led to that cost?"
The three questions above cannot be answered by standard decision theory, where the goal is to produce a single policy that minimizes expected cost. These questions also cannot be answered by robust optimization, where the goal is to produce a single policy that is robust to the uncertainty in nature. Those paradigms produce a single policy decision that takes uncertainty into account, and the chosen policy might not be a best response policy to any realistic situation. In contrast, our goal is to understand the uncertainty and how to react to it, using policies that would be best responses to individual situations. 
}%COLOR

\textcolor{Black}{ %COLOR FIRST LAYER
There are many applications in which this method can be used. For example, in scheduling staff for a medical clinic, predictions based on a statistical model of the number of patients might be used to understand the possible policies and costs for staffing. In traffic flow problems, predictions based on a model of the forecasted traffic might be useful for determining load balancing policies on the network and their associated costs. In online advertising, predictions based on models for the payoff and ad-click rate might be used to understand policies for when the ad should be displayed and the associated revenue. 
}%COLOR

\textcolor{Black}{ %COLOR FIRST LAYER
In order to propagate the uncertainty in modeling to the uncertainty in costs, we introduce what we call the \textit{simultaneous
process}, where we explore the range of predictive models and corresponding policy decisions at the
same time. The simultaneous process was named to contrast with a more traditional \textit{sequential process}, where
first, data are input into a statistical algorithm to produce a predictive model, which makes recommendations
for the future, and second, the user develops a plan of action and projected cost for implementing the policy. The sequential process is commonly used
in practice, even though there may actually be a whole class of models that could be relevant for
the policy decision problem. The sequential process essentially assumes that the probabilistic model is ``correct enough" to make a decision that is ``close enough."  
}%COLOR

\textcolor{Black}{ %COLOR FIRST LAYER
In the simultaneous process, the machine learning algorithm contains a regularization term encoding the
policy and its associated cost, with an adjustable regularization parameter.  If there is some uncertainty about how much it will cost to solve the problem, the regularization
parameter can be swept through an interval to find a range of possible costs, from optimistic to
pessimistic. The method then produces the most likely scenario for each value of the cost. This way, by looking at the full range of the regularization parameter, we sweep out costs for all of the reasonable probabilistic models. This range can be used to determine how much might be reasonably allocated to solve the problem. 
}%COLOR

\textcolor{Black}{ %COLOR FIRST LAYER
Having the full range of costs for reasonable models can directly answer the question in the first paragraph regarding allocation, ``What is a reasonable amount to allocate for this task so we can react best to whatever nature brings?" One might choose to allocate the maximum cost for the set of reasonable predictive models for instance.
%Ideally, one would want to consider all reasonable scenarios, not just the most likely ones, but this is computational infeasible here.)
%Understanding the reasonable range of operational costs can provide a substantial benefit: for instance, it is possible that a slightly different predictive model can induce a large change in the operational cost, without decreasing predictive quality. Let us assume temporarily that this is the case, and see how the three questions in the first paragraph can be answered. Indeed, knowing that there is a wide range of reasonable predictive models with very different operational costs can help with allocation problems, as in the first question above, regarding allocation. One can use the simultaneous process to locate a reasonable predictive model that has high operational cost, and allocate the amount corresponding to that model. This amount should be robust to whichever predictive model is (close to being) realized.
The second question above is ``Can we produce a reasonable probabilistic model, supported by data, where we might expect to pay a specific amount?" This is an important question, since business managers often like to know if there is some scenario/decision pair that is supported by the data, but for which the operational cost is low (or high); the simultaneous process would be able to find such scenarios directly. To do this, we would look at the setting of the regularization parameter that resulted in the desired value of the cost, and then look at the solution of the simultaneous formulation, which gives the model and its corresponding policy decision. 
}%COLOR

\textcolor{Black}{ %COLOR FIRST LAYER
Let us consider the third question above, which is ``Can our intuition about how much it will cost to solve a problem help us produce a better probabilistic model?" The regularization parameter can be interpreted to regulate the strength of our belief in the operational cost. If we have a strong belief in the cost to solve the problem, and if that belief is correct, this will guide the choice of regularization parameter, and will help with prediction. In many real scenarios, a practitioner or domain expert might truly have a prior belief on the cost to complete a task. Arguably, a manager having this more grounded type of prior belief is much more natural than, for instance, the manager having a prior belief on the $\ell_2$ norm of the coefficients of a linear model, or the number of nonzero coefficients in the model. Being able to encode this type of prior belief on cost could potentially be helpful for prediction: as with other types of prior beliefs, it can help to restrict the hypothesis space and can assist with generalization. In this work, we show that the restricted hypothesis spaces resulting from our method can often be bounded by an intersection of an an $\ell_{q}$  ball with a halfspace - and this is true for many different types of decision problems. We analyze the complexity of this type of hypothesis space with a technique based on Maurey's Lemma \citep{barron93,zhang02} that leads eventually to a counting problem, where we calculate the number of integer points within a polyhedron in order to obtain a covering number bound.
}%COLOR

\textcolor{Black}{ %COLOR FIRST LAYER
The operational cost regularization term can be the optimal value of a complicated optimization
problem, like a scheduling problem. This means we will need to solve an optimization problem
each time we evaluate the learning algorithm's objective. However, the practitioner must be able to
solve that problem anyway in order to develop a plan of action; it is the same problem they need
to solve in the traditional sequential process, or using standard decision theory. Since the decision problem is solved only on data
from the present, whose labels are not yet known, solving the decision problem may not be difficult,
especially if the number of unlabeled examples is small. In that case, the method can still scale up
to huge historical data sets, since the historical data factors into the training error term but not the new
regularization term, and both terms can be computed. An example is to compute a schedule for a day, based on factors of the various meetings on the schedule that day. We can use a very large amount of past meeting-length data for the training error term, but then we use only the small set of possible meetings coming up that day to pass into the scheduling problem. In that case, both the training error term and the regularization term are able to be computed, and the objective can be minimized. 
}%COLOR

%The new simultaneous process also contrasts with traditional decision making approaches such as Bayesian decision theory, minimax and maximax decision theory, which also do not help us to understand the uncertainty regime. These methods either assume known randomness or a known range of possible predictions and make a single decision, whereas our method directly takes data, translates them into a set of predictive models, each with an associated decision and cost.

\textcolor{Black}{ %COLOR SECOND LAYER
The simultaneous process is a type of decision theory. To give some background, there are two types of relevant decision theories: normative (which assumes full information, rationality and infinite computational power) and descriptive (models realistic human behavior).
Normative decision theories that address decision making under uncertainty can be classified into those based on ignorance (using no probabilistic information) and those based on risk (using probabilistic information). The former include maximax, maximin (Wald), minimax regret (Savage), criterion of realism (Hurwicz), equally likely (Laplace) approaches. The latter include utility based expected value and bayesian approaches (Savage). Info-gap, Dempster-Shafer, fuzzy logic, and possibility theories offer non-probabilistic alternatives to probability in Bayesian/expected value theories \citep{french1986decision,hansson1994decision}.
}

\textcolor{Black}{ %COLOR SECOND LAYER
The simultaneous process does not fit into any of the decision theories listed above.  
For instance, a core idea in the Bayesian approach is to choose a single policy that maximizes expected utility, or minimizes expected cost. Our goal is not to find a single policy that is useful on average. In contrast, our goal is to trace out a path of models, their specific (not average) optimal-response policies, and their costs.  
The policy from the Bayesian approach may not correspond to the best decision for any particular single model, whereas that is something we want in our case. We trace out this path by changing our prior belief on the operational cost (that is, by changing the strength of our regularization term). In Bayesian decision theory, the prior is over possible probabilistic models, rather than on possible costs as in this paper. Constructing this prior over possible probabilistic models can be challenging, and the prior often ends up being chosen arbitrarily, or as a matter of convenience. In contrast, we assume only an unknown probability measure over the data, and the data itself defines the possible probabilistic models for which we compute policies. 
}
%\textcolor{Black}{ %COLOR FIRST LAYER
%As mentioned above, the simultaneous process is different from decision theory, for instance, Bayesian decision theory. A core idea in decision theory is to choose a single policy that maximizes expected utility, or minimizes expected cost. Our goal is not to find a single policy that is useful on average. In contrast, our goal is to trace out a path of models, their specific (not average) optimal-response policies, and their costs.  
%The policy from decision theory may not correspond to the best decision for any particular single model, whereas that is something we want in our case. We trace out this path by changing our prior belief on the operational cost (that is, by changing the strength of our regularization term). In Bayesian decision theory, the prior is over possible probabilistic models, rather than on possible costs as in this paper.
%%In decision theory, the goal is to use our understanding of uncertainty to produce a single good decision. Here, our goal is to produce a range of possible good decisions. 
%(Note that one could potentially use Bayesian decision theory to produce a posterior, by incorporating a prior belief on the operational cost as in this paper, but that is not our goal in this work.) 
%}%COLOR

\textcolor{Black}{ %COLOR SECOND LAYER
Maximax (optimistic) and maximin (pessimistic) decision approaches contrast with the Bayesian framework and do not assume a distribution on the possible probabilistic models. In Section \ref{sec:robust} we will discuss how these approaches are related to the simultaneous process. They overlap with the simultaneous process but not completely. Robust optimization is a maximin approach to decision making, and the simultaneous process also differs in principle from robust optimization. In robust optimization, one would generally need to allocate much more than is necessary for any single realistic situation, in order to produce a policy that is robust to almost all situations. However, this is not always true; in fact, we show in this work that in some circumstances, while sweeping through the regularization parameter, one of the results produced by the  simultaneous process is the same as the one coming from robust optimization. 
}

\textcolor{Black}{ %COLOR FIRST LAYER
We introduce the sequential and simultaneous processes in Section \ref{sec:formulation}.
In Section \ref{sec:experiments}, we give several examples of algorithms that incorporate these operational costs.
\textcolor{Black}{ %COLOR SECOND LAYER
In doing so, we provide answers for the first two questions Q1 and Q2 above, with respect to specific problems.}
%What is a reasonable amount to allocate for the task? and if we can produce good models which lead to specific costs.
Our first example application is a staffing problem at a medical clinic, where the decision problem is to staff a set of stations that patients must complete in a certain order. The time required for patients to complete each station is random and estimated from past data. The second example is a real-estate purchasing problem, where the policy decision is to purchase a subset of available properties. The values of the properties need to be estimated from comparable sales. The third example is a call center staffing problem, where we need to create a staffing policy based on historical call arrival and service time information. A fourth example is the ``Machine Learning and Traveling Repairman Problem" (ML\&TRP) where the policy decision is a route for a repair crew.
As mentioned above, there is a large subset of problems that can be formulated using the simultaneous process that have a special property: they are equivalent to robust optimization (RO) problems.
Section \ref{sec:robust} discusses this relationship and provides, under specific conditions, the equivalence of the simultaneous process with RO. Robust optimization, when used for decision-making, does not usually include machine learning, nor any other type of statistical model, so we discuss how a statistical model can be incorporated within an uncertainty set for an RO. Specifically, we discuss how different loss functions from machine learning correspond to different uncertainty sets. We also discuss the overlap between RO and the optimistic and pessimistic versions of the simultaneous process.
}%COLOR

We consider the implications of the simultaneous process on statistical learning theory in Section \ref{sec:bound}. In particular, we aim to understand how operational costs affect prediction (generalization) ability.
\textcolor{Black}{ %COLOR SECOND LAYER
This helps answer the third question Q3, about how intuition about operational cost can help produce a better probabilistic model.}
%The simultaneous process essentially introduces a bias towards low or high operational cost, where ``bias" means (as usual) a preference for certain desirable properties 
%(\textit{e.g.,} another type of bias is model sparsity \citep{tibshirani96}). 
%\citep[\textit{e.g.,} another type of bias is model sparsity][]{tibshirani96}. 
We show first that the hypothesis spaces for most of the applications in Section \ref{sec:experiments} can be bounded in a specific way - by an intersection of a ball and a halfspace - and this is true regardless of how complicated the constraints of the optimization problem are, and how different the operational costs are from each other in the different applications. Second, we bound the complexity of this type of hypothesis space using a technique based on Maurey's Lemma \citep{barron93,zhang02} that leads eventually to a counting problem, where we calculate the number of integer points within a polyhedron in order to obtain a generalization bound.
Our results show that it is possible to make use of much more general structure in estimation problems, compared to the standard (norm-constrained) structures like sparsity and smoothness; further, this additional structure can benefit generalization ability. A shorter version of this work has been previously published \citep[see][]{TuRu12isaim}.

%We start by formalizing the ``simultaneous" process, where operational costs are incorporated into machine learning algorithms.

%%%%%%%%%%%%Section: The Sequential and Simultaneous Processes
\section{The Sequential and Simultaneous Processes}\label{sec:formulation}

We have a training set of (random) labeled instances, $\{(x_i, y_i)\}_{i=1}^n$, where $x_i\in\X$, $y_i\in\Y$ that we will use to learn a function $f^*:\X\rightarrow \Y$. Commonly in machine learning this is done by choosing $f$ to be the solution of a minimization problem:
\begin{equation}
f^*\in\argmin_{f\in\F^{unc}}\left( \sum_{i=1}^{n} l(f(x_i),y_i)+C_{2}R(f)\right),
\label{eqn:reg-train-loss}
\end{equation}
for some loss function $l:\Y\times\Y \rightarrow \mathbb{R}_{+}$, regularizer $R:\F^{unc} \rightarrow \mathbb{R}$, constant $C_{2}$ and function class $\F^{unc}$. Here, $ \Y\subset \mathbb{R}$. Typical loss functions used in machine learning are the 0-1 loss, ramp loss, hinge loss, logistic loss and the exponential loss. Function class $\F^{unc}$ is commonly the class of all linear functionals, where an element $f \in \mathcal{F}^{unc}$ is of the form $\beta^{T}x$, where $\X\subset \mathbb{R}^{p}$, $\beta \in \mathbb{R}^{p}$. We have used `$unc$' in the superscript for $\F^{unc}$ to refer to the word ``unconstrained,'' since it contains all linear functionals. Typical regularizers $R$ are the $\ell_{1}$ and $\ell_{2}$ norms of $\beta$. 
Note that nonlinearities can be incorporated into $\F^{unc}$ by allowing nonlinear features, so that we now would have $f(x)=\sum_{j=1}^p\beta_jh_j(x)$, where $\{h_j\}_j$ is the set of features,
which can be arbitrary nonlinear functions of $x$; for simplicity in notation, we will equate $h_j(x)=x_j$ and have  $\X\subset \mathbb{R}^{p}$.

Consider an organization making policy decisions. Given a new collection of unlabeled instances $\{\tilde{x}_i\}_{i=1}^m$, the organization wants to create a policy $\pi^*$ that minimizes a certain operational cost $\Obj(\pi,f^*,\{\tilde{x}_i\}_i)$. Of course, if the organization knew the true labels for the $\{\tilde{x}_i\}_i$'s beforehand, it would choose a policy to optimize the operational cost based directly on these labels, and would not need $f^*$. Since the labels are not known, the operational costs are calculated using the model's predictions, the $f^*(\tilde{x}_i)$'s. The difference between the traditional sequential process and the new simultaneous process is whether $f^*$ is chosen with or without knowledge of the operational cost.  

As an example, consider $\{\xt_i\}_{i}$ as representing machines in a factory waiting to be repaired, where the first feature $\xt_{i,1}$ is the age of the machine, the second feature $\xt_{i,2}$ is the condition at its last inspection, etc. The value $f^*(\xt_i)$ is the predicted probability of failure for $\xt_i$. Policy $\pi^*$ is the order in which the machines $\{\xt_i\}_i$ are repaired, which is chosen based on how likely they are to fail, that is, $\{f^*(\xt_i)\}_i$, and on the costs of the various types of repairs needed. The traditional sequential process picks a model $f^*$, based on past failure data without the knowledge of operational cost, and afterwards computes $\pi^*$ based on an optimization problem involving the $\{f^{*}(\xt_i)\}_{i}$'s and the operational cost. The new simultaneous process picks $f^{*}$ and $\pi^*$ at the same time, based on optimism or pessimism on the operational cost of $\pi^*$.

Formally, the \textbf{sequential process} computes the policy according to two steps, as follows.
\begin{description}
\item [Step 1:] Create function $f^{*}$ based on $\{(x_i, y_i)\}_i$ according to (\ref{eqn:reg-train-loss}). That is 
\[f^*\in\argmin_{f\in\F^{unc}} \left(\sum_{i=1}^{n} l(f(x_i),y_i)+C_{2}R(f)\right).\]
\item [Step 2:] Choose policy $\pi^*$ to minimize the operational cost, 
\[\pi^*\in\argmin_{\pi\in\Pi} \Obj(\pi,f^*,\{\tilde{x}_{i}\}_i).\] 
\end{description}
The operational cost $\Obj(\pi,f^*,\{\tilde{x}_{i}\}_i)$ is the amount the organization will spend if policy $\pi$ is chosen in response to the values of $\{f^*(\tilde{x}_{i})\}_{i}$.

To define the \textbf{simultaneous process}, we combine Steps 1 and 2 of the sequential process. We can choose an \textbf{optimistic bias}, where we prefer (all else being equal) a model providing lower costs, or we can choose a \textbf{pessimistic bias} that prefers higher costs, where the degree of optimism or pessimism is controlled by a parameter $C_1$. in other words, the optimistic bias lowers costs when there is uncertainty, whereas the pessimistic bias raises them. The new steps are as follows. 
\begin{description}
\small
\item [Step 1:]  Choose a model $f^{\circ}$ obeying one of the following:
\begin{align}\nonumber
\textrm{Optimistic}&\textrm{ Bias: }f^{\circ} \in \argminf \left[ \sum_{i=1}^{n} l\left(f(x_i),y_i\right) \right.\\
&\hspace*{-20pt} \left. +C_2 R(f) +C_1 \minpi \Obj\left(\pi,f,\{\tilde{x}_{i}\}_i\right) \right]\hspace*{0pt}\label{eqn:optimisticbias}\\\nonumber
\textrm{Pessimistic}&\textrm{ Bias: } f^{\circ} \in \argminf \left[ \sum_{i=1}^{n} l\left(f(x_i),y_i\right) \right.\\
&\hspace*{-20pt} \left.+C_2 R(f) -C_1 \minpi \Obj\left(\pi,f,\{\tilde{x}_{i}\}_i\right)  \right]\hspace*{-3pt}.\hspace*{0pt}\label{eqn:pessimisticbias}
\end{align}
\item [Step 2:] Compute the policy: $$\pi^{\circ} \in \argminpi \Obj\left(\pi,f^{\circ},\{\tilde{x}_{i}\}_i\right).$$
\end{description}

\textcolor{Black}{ %COLOR FIRST LAYER
When $C_1=0$, the simultaneous process becomes the sequential process; the sequential process is a special case of the simultaneous process.
}%COLOR

\textcolor{Black}{ %COLOR FIRST LAYER
The optimization problem in the simultaneous process can be computationally difficult, particularly if the subproblem to minimize $\Obj$ involves discrete optimization. However, if the number of unlabeled instances is small, or if the policy decision can be broken into several smaller subproblems, then even if the training set is large, one can solve Step 1 using different types of mathematical programming solvers, including MINLP solvers \citep{bonmin}, Nelder-Mead \citep{neldermead} and Alternating Minimization schemes \citep{TuRuJaadt11}. One needs to be able to solve instances of that optimization problem in any case for Step 2 of the sequential process. The simultaneous process is more intensive than the sequential process in that it requires repeated solutions of that optimization problem, rather than a single solution.
}%COLOR
 
The regularization term $R(f)$ can be for example, an $\ell_1$ or $\ell_2$ regularization term to encourage a sparse or smooth solution. 

\textcolor{Black}{ %COLOR FIRST LAYER
As the $C_1$ coefficient swings between large values for optimistic and pessimistic cases, the algorithm finds the best solution (having the lowest loss with respect to the data) for each possible cost. Once the regularization coefficient is too large, the algorithm will sacrifice empirical error in favor of lower costs, and will thus obtain solutions that are not reasonable.
%Once the regularization coefficient is too large, we are no longer taking into account the learning error and as a result obtain trivial low cost solutions which are easy to ascertain. 
%Once the regularization coefficient is too large, the algorithm will no longer be able to produce reasonably low cost solutions. 
When that happens, we know we have already mapped out the full range of costs for reasonable solutions. This range can be used for pre-allocation decisions. 
}%COLOR
%As with sparse or smooth regularization terms, the new operational cost bias can be interpreted as a prior belief about the model - in this case, a belief that the operating costs should be lower or higher on the current set of unlabeled instances $\{\tilde{x}_{i}\}_{i}$. In that sense, we introduce a type of regularization that may have a closer connection to reality than typical (e.g., $\ell_{1}$ or $\ell_{2}$ norm) regularizers. If one asks a manager at a company what prior belief they have about the estimation model, it is not likely they would give a answer in terms of coefficients for a linear model. Even managers who are not mathematicians or computer scientists might have some belief - they could perhaps believe that they are expecting to spend a certain amount to enact the policy. It is possible that this type of belief, which relies on direct experience, might be more practical, and more accurate, than the more abstract prior information that we are typically used to dealing with. Further, the simultaneous process can be used to assist in pre-allocating costs. If there is some uncertainty about how much it will cost to solve a problem, the simultaneous process can be used to find a range of possible costs, from optimistic to pessimistic, which will determine how much should be allocated to solve the problem.

\textcolor{Black}{ %COLOR SECOND LAYER
By sweeping over a range of $C_{1}$, we obtain a range of costs that we might incur. Based on this range, we can choose to allocate a reasonable amount of resources so that we can react best to whatever nature brings. This helps answer question Q1 in Section \ref{sec:intro}. In addition, we can pick a value of $C_{1}$ such that the resulting operational cost is a specific amount. In this case, we checking whether a probabilistic model exists, corresponding to that cost, that is reasonably supported by data. This can answer question Q2 in Section \ref{sec:intro}.
}

 It is possible for the set of feasible policies $\Pi$ to depend on recommendations $\{f(\tilde{x}_{1}),...,f(\tilde{x}_{m})\}$, so that $\Pi=\Pi(f,\{\tilde{x}_{i}\}_i)$ in general. We will revisit this possibility in Section \ref{sec:robust}. It is also possible for the optimization over $\pi\in\Pi$ to be trivial, or the optimization problem could have a closed form solution. Our notation does accommodate this, and is more general.

%We would like to clarify some things to avoid possible misperceptions about the general idea. First, we are not claiming that ``truth" is altered by what one needs to do with it; one can view the operational cost term purely as encoding a prior belief about the truth. The prior belief happens to be about the operational cost.  Second, what we call ``operational cost" is essentially a form of utility that is used to bias the picture of the world towards anticipated decisions. One should not view the operating cost as a utility function that needs to be estimated, as in reinforcement learning, where we do not know the cost. It is possible to extend our framework to estimate the utility, but currently, the cost is fixed and there is no separate utility: one knows precisely what the cost will be under each possible outcome. For instance, if we are estimating prices, and then the price is revealed, we know exactly what we will pay.

One should not view the operational cost as a utility function that needs to be estimated, as in reinforcement learning, where we do not know the cost. Here one knows precisely what the cost will be under each possible outcome. Unlike in reinforcement learning, we have a complicated one shot decision problem at hand and have training data as well as future/unlabeled examples on which the predictive model makes prediction on. 

\textcolor{Black}{ %COLOR FIRST LAYER
The use of unlabeled data $\{\tilde{x}_{i}\}_{i}$ has been explored widely in the machine learning literature under semi-supervised, transductive, and unsupervised learning. In particular, we point out that the simultaneous process is not a semi-supervised learning method \citep[see][]{ChaSchZie06}, since it does not use the unlabeled data to provide information about the underlying distribution. A small unlabeled sample is not very useful for semi-supervised learning, but could be very useful for constructing a low-cost policy. The simultaneous process also has a resemblance to transductive learning \citep[see][]{Zhu07}, whose goal is to produce the output labels on the set of unlabeled examples; in this case, we produce a function (namely the operational cost) applied to those output labels. The simultaneous process, for a fixed choice of $C_1$, can also be considered as a multi-objective machine learning method, since it involves an optimization problem having two terms with competing goals \citep[see][]{Jin06}. 
}%COLOR

%%%% sub-section: SRM
\subsection{The Simultaneous Process in the Context of Structural Risk Minimization}\label{subsec:srm}
%In predictive modeling problems, there is often no one right statistical model when dealing with finite datasets. 
%Using the traditional (sequential) method for making decisions, if there are many good models for describing the data, we usually choose one that has certain desirable properties, for instance the ``simplest" or most parsimonious model, where ``simple" has many different definitions. Or, the chosen model is somehow in agreement with a prior belief, though in most cases, there is no completely justifiable prior belief about what the outcome probabilities might be, in other words, there is no strong evidence for the prior belief.  
%Either way, there could be a whole class of approximately equally good models, and there is uncertainty regarding which model to choose. 
%The desirable properties of a model (like simplicity) tend to indicate that it will predict well on future instances drawn from the same distribution.

In the framework of statistical learning theory 
\citep[e.g.,][]{Vapnik98,Pollard84,Bartlett99,zhang02},
prediction ability of a class of models is guaranteed when the class has low ``complexity," where complexity is defined via covering numbers, VC (Vapnik-Chervonenkis) dimension, Rademacher complexity, gaussian complexity, etc. Limiting the complexity of the hypothesis space imposes a bias, and the classical image associated with the bias-variance tradeoff is provided in Figure \ref{FigureGeneraliz}(a).  
\begin{figure}
 \begin{center}
  \resizebox{270pt}{!} {
  \includegraphics{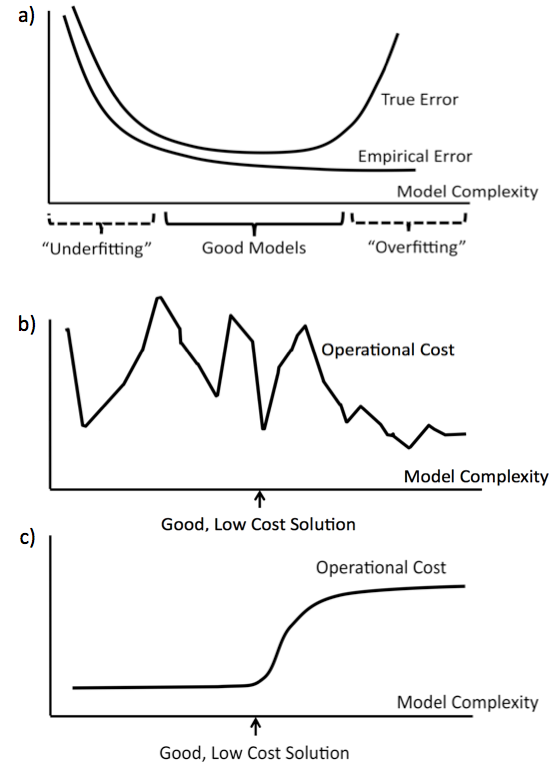}}
\end{center}
 \caption{\label{FigureGeneraliz} In all three plots, the x-axis represents model classes with increasing complexity. a) Relationship between training error and test error as a function of model complexity. b) A possible operational cost as a function of model complexity. c) Another possible operational cost.} 
\end{figure}
The set of good models is indicated on the axis of the figure. Models that are not good are either overfitted (explaining too much of the variance of the data, having a high complexity), or underfitted (having too strong of a bias and a high empirical error).
By understanding complexity, we can find a model class where both the training error and the complexity are kept low. An example of increasingly complex model classes is the set of nested classes of polynomials, starting with constants, then linear functions, second order polynomials and so on.

\textcolor{Black}{ %COLOR FIRST LAYER
In predictive modeling problems, there is often no one right statistical model when dealing with finite datasets, in fact there may be a whole class of good models. In addition,
it is possible that a small change in the choice of predictive model could lead to a large change in the cost required to implement the policy recommended by the model. This occurs, for instance, when costs are based on  objects (e.g., products) that come in discrete amounts.
Figure \ref{FigureGeneraliz}(b) illustrates this possibility, by showing that there may be a variety of costs amongst the class of good models. The simultaneous process can find the range of costs for the set of good models, which can be used for allocation of costs, as discussed in the first question Q1 in the introduction.
}%COLOR
%Figure  \ref{FigureGeneraliz}(b) exemplifies the cases in which the simultaneous process can have a dramatic impact on the solution: it is when the training error is relatively flat near its minimizers, and in that same region, the operational cost has large variations. This situation arises naturally when we are very uncertain about quantities influencing the cost, either because data are scarce or noisy, the problem's dimensionality is large, or because the operational cost is not smooth. 
%Consider the following hypothetical problem: suppose there are two statistical models that are both equally good: their performance is approximately the same on the sample (same empirical error), and for these two models, our prior knowledge about the problem says they will perform approximately equally well on a held-out test set. These models are equally parsimonious, and equivalent in every other way with respect to our belief in their prediction quality. Then which do we pick It seems that it should not matter which one, if the end goal is simply prediction. In this work, we consider a different end goal, which is how the models are used in practice. What if one of these models is quite a bit more expensive to implement than the other one In that case, there can be a substantial benefit to picking the more cost effective solution. Many of the cost functions used in practice are only piecewise continuous, for example, allocation of resources that come in discrete amounts. So 

\textcolor{Black}{ %COLOR SECOND LAYER
Recall that question Q3 asked if our intuition about how much it will cost to solve a problem can help us produce a better probabilistic model.
Figure \ref{FigureGeneraliz} can be used to illustrate how this question can be answered. Assume we have a strong prior belief that the operational cost will not be above a certain fixed amount.}
Accordingly, we will choose only amongst the class of low cost models. This can significantly limit the complexity of the hypothesis space, because the set of low-cost good models might be much smaller than the full space of good models. Consider, for example, the cost displayed in Figure \ref{FigureGeneraliz}(c), where only models on the left part of the plot would be considered, since they are low cost models. Because the hypothesis space is smaller, we may be able to produce a tighter bound on the complexity of the hypothesis space, thereby obtaining a better prediction guarantee for the simultaneous process than for the sequential process. In Section \ref{sec:bound} we develop results of this type. These results indicate that in some cases, the operational cost can be an important quantity for generalization.

%The first hypothetical situation discussed above corresponds to the choice that tradeoff parameter $C_1$ is small (and $C_2$ is large enough to prevent overfitting). The second hypothetical situation corresponds to where $C_1$ is large enough to impact generalization (and $C_2$ is again sufficiently large).  
%In neither of these hypothetical situations does the simultaneous process go against the principle of Occam's Razor, which prefers the ``simplest" model that makes the least new assumptions to explain the data. In the first situation, we choose the most practical among equally simple models. In the second hypothetical situation, we make only one new assumption to explain the data - that the operational cost is low. 
%In the second hypothetical situation, because the hypothesis space is limited to low-cost solutions, we may be able to produce a tighter bound on the complexity of the hypothesis space, 

%%%%%%%%%%%%%%%%%%%%%Section:Conceptual Demonstrations
\section{Conceptual Demonstrations}\label{sec:experiments}
%2011-12-02: do the opposite bias for housing data, manpower, call-ctr

We provide four examples. In the first, we estimate manpower requirements for a scheduling task. In the second, we estimate real estate prices for a purchasing decision. In the third, we estimate call arrival rates for a call center staffing problem. In the fourth, we estimate failure probabilities for manholes (access points to an underground electrical grid). The first two are small scale reproducible examples, designed to demonstrate new types of constraints due to operational costs. In the first example, the operational cost subproblem involves scheduling. In the second, it is a knapsack problem, and in the third, it is another multidimensional knapsack variant. In the fourth, it is a routing problem. In the first, second and fourth examples, the operational cost leads to a linear constraint, while in the third example, the cost leads to a quadratic constraint.

Throughout this section, we will assume that we are working with linear functions $f$ of the form $\beta^{T}x$ so that $\Pi(f,\{\tilde{x}_{i}\}_{i})$ can be denoted by $\Pi(\beta,\{\tilde{x}_{i}\}_{i})$. We will set $R(f)$ to be equal to $\|\beta\|_{2}^{2}$.
We will also use the notation $\Fr$ to denote the set of linear functions that satisfy an additional property:
\[
\Fr := \{f \in \F^{unc}: R(f) \leq C_{2}^{*}\},
\]
where $C_{2}^{*}$ is a known constant greater than zero. We will use constant $C_{2}$ from (\ref{eqn:reg-train-loss}), and also $C_{2}^{*}$ from the definition of $\Fr$, to control the extent of regularization. $C_{2}$ is inversely related to $C_{2}^{*}$. We use both versions interchangeably throughout the paper.

%%%Manpower and Scheduling
\subsection{Manpower Data and Scheduling with Precedence Constraints} 
We aim to schedule the starting times of medical staff, who work at 6 stations, for instance, ultrasound, X-ray, MRI, CT scan, nuclear imaging, and blood lab. Current and incoming patients need to go through some of these stations in a particular order. The six stations and the possible orders are shown in Figure \ref{fig:problem_instance1}. Each station is denoted by a line. Work starts at the check-in (at time $\pi_1$) and ends at the check-out (at time $\pi_5$). The stations are numbered 6-11, in order to avoid confusion with the times $\pi_1$-$\pi_5$. The clinic has precedence constraints, where a station cannot be staffed (or work with patients) until the preceding stations are likely to finish with their patients. For instance, the check-out should not start until all the previous stations finish. Also, as shown in Figure \ref{fig:problem_instance1}, station 11 should not start until stations 8 and 9 are complete at time $\pi_4$, and station 9 should not start until station 7 is complete at time $\pi_3$. Stations 8 and 10 should not start until station 6 is complete.
(This is related to a similar problem called \textit{planning with preference} posed by F. Malucelli, Politecnico di Milano). 
\begin{figure}[t]
\centering
\includegraphics[width=0.4\textwidth]{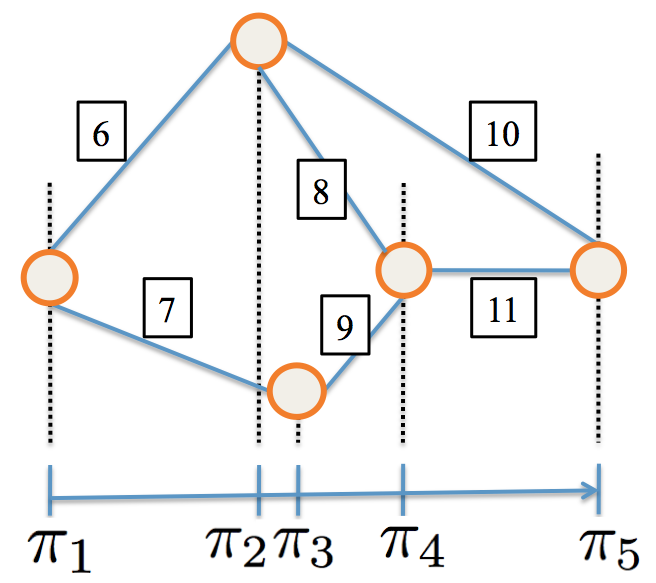}
\caption{Staffing estimation with bias on scheduling with precedence constraints.\label{fig:problem_instance1}}
\end{figure}

The operational goal is to minimize the total time of the clinic's operation, from when the check-in happens at time $\pi_1$ until the check-out happens at time $\pi_5$.
We estimate the time it takes for each station to finish its job with the patients based on two variables: the new load of patients for the day at the station, and the number of current patients already present. The data are available as \textit{manpower} in the R-package \textit{bestglm}, using ``Hour," ``Load" and ``Stay" columns. The training error is chosen to be the least squares loss between the estimated time for stations to finish their jobs ($\beta^{T}x_{i}$) and the actual times it took to finish ($y_{i}$). The unlabeled data are the new load and current patients present at each station for a given period, given as $\tilde{x}_{6},..,\tilde{x}_{11}$. Let $\pi$ denote the $5$-dimensional real vector with coordinates $\pi_{1},...,\pi_{5}$.

The operational cost is the total time $\pi_5-\pi_1$. Step 1, with an optimistic bias, can be written as:
\begin{equation}
\min_{\{\beta:\|\beta\|_{2}^{2}\leq C_{2}^{*}\}}\sum_{i=1}^{n}(y_{i}  -\beta^{T}x_{i})^{2} +C_{1}\min_{\pi \in \Pi(\beta,\{\tilde{x}_{i}\}_{i})} (\pi_{5} - \pi_{1} ),
\label{eqn:problem_instance1}
\end{equation}
where the feasible set $\Pi(\beta,\{\tilde{x}_{i}\}_{i})$ is defined by the following constraints:
\begin{align*}
\pi_{a}+ \beta^{T}\tilde{x}_{i} \leq \pi_{b}  ;& \;\; (a,i,b) \in \{(1,6,2),(1,7,3),(2,8,4),(3,9,4),(2,10,5),(4,11,5)\}\\
\pi_{a} \geq 0 & \textrm{ for } a = 1,...,5.
\end{align*}
To solve (\ref{eqn:problem_instance1}) given values of $C_{1}$ and $C_{2}$, we used a function-evaluation-based scheme called Nelder-Mead \citep{neldermead} where at every iterate of $\beta$, the subproblem for $\pi$ was solved to optimality (using Gurobi\footnote{Gurobi Optimizer v3.0, Gurobi Optimization, Inc. 2010.}). $C_{2}$ was chosen heuristically based on (\ref{eqn:reg-train-loss}) and kept fixed for the experiment beforehand.

Figure \ref{fig:problem_instance1_results} shows the operational cost, training loss, and $r^2$ statistic\footnote{If $\hat{y}_{i}$ are the predicted labels and $\bar{y}$ is the mean of $\{y_{1},...,y_{n}\}$, then the value of the $r^{2}$ statistic is defined as $1 - \sum_{i}(y_{i}-\hat{y}_{i})^{2}/\sum_{i}(y_{i}-\bar{y})^{2}$. Thus $r^2$ is an affine transformation of the sum of squares error. $r^2$ allows training and test accuracy to be measured on a comparable scale.} for various values of $C_1$. For $C_1$ values between $0$ and $0.2$, the operational cost varies substantially, by $\sim$16\%. The $r^2$ values for both training and test vary much less, by $\sim$3.5\%, where the best value happened to have the largest value of $C_1$. For small datasets, there is generally a variation between training and test:  for this data split, there is a 3.16\% difference in $r^2$ between training and test for plain least squares, and this is similar across various splits of the training and test data. 
\textcolor{Black}{ %COLOR FIRST LAYER
This means that for the scheduling problem, there is a range of reasonable predictive models within about $\sim$3.5\% of each other. 
}%COLOR

\textcolor{Black}{ %COLOR SECOND LAYER
What we learn from this, in terms of the three questions in the introduction, is that: 1) There is a wide range of possible costs within the range of reasonable optimistic models. 2) We have found a reasonable scenario, supported by data, where the cost is 16\% lower than in the sequential case. 3) If we have a prior belief that the cost will be lower, the models that are more accurate are the ones with lower costs, and therefore we may not want to designate the full cost suggested by the sequential process. We can perhaps designate up to 16\% less.}%COLOR
%, and we may produce a better since we believe we can get away with spending possibly up to 16\% less.
%The optimistic bias allows a more cost-effective solution, staying within this regime of reasonable predictive models.  If we are lucky, we can potentially save $16\%$ of operational time.

\begin{figure}
\centering 
     \includegraphics[width=.33\textwidth]{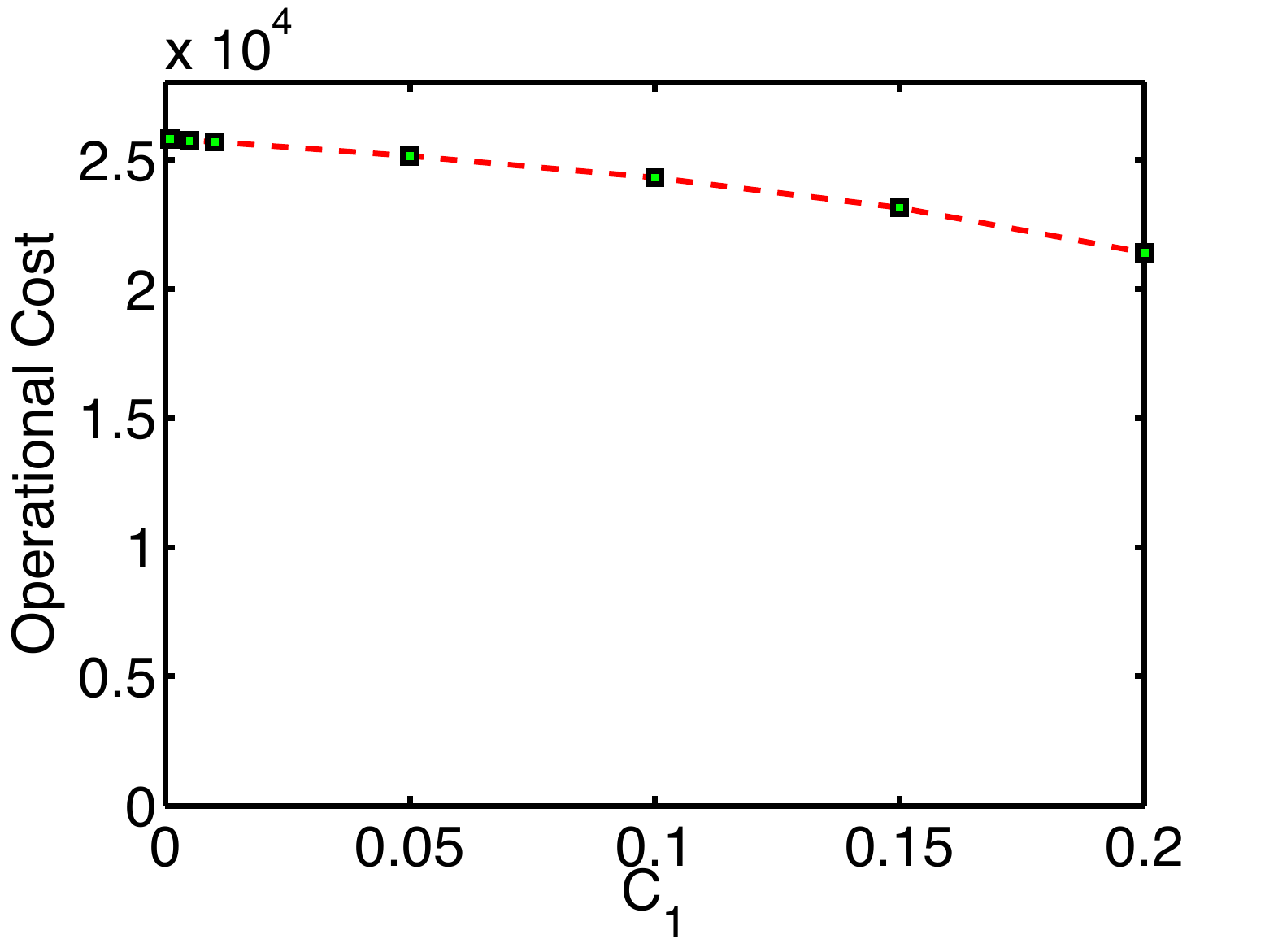}
\hspace{-15pt}
	\includegraphics[width=.33\textwidth]{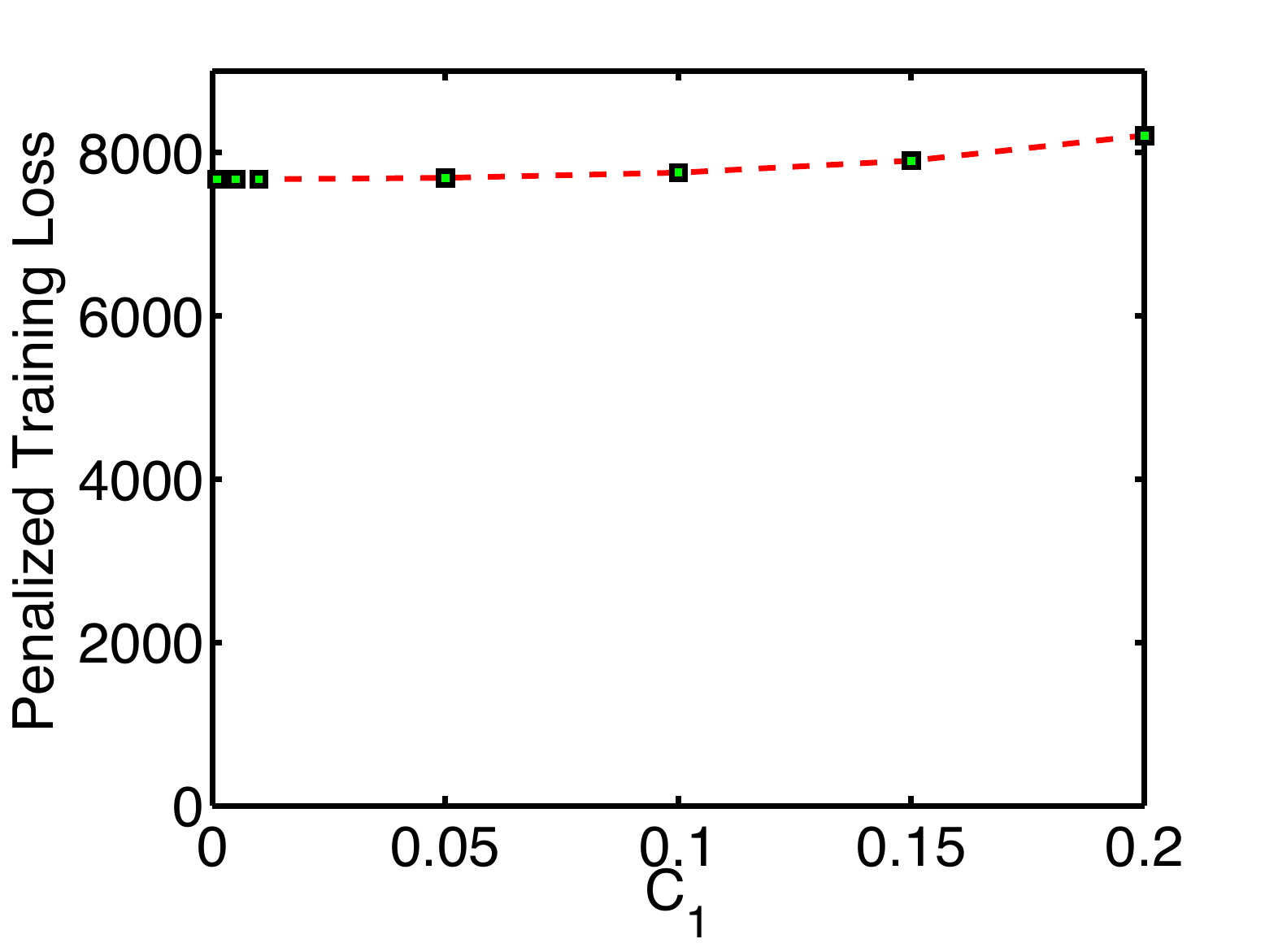}
\hspace{-15pt}
	\includegraphics[width=.33\textwidth]{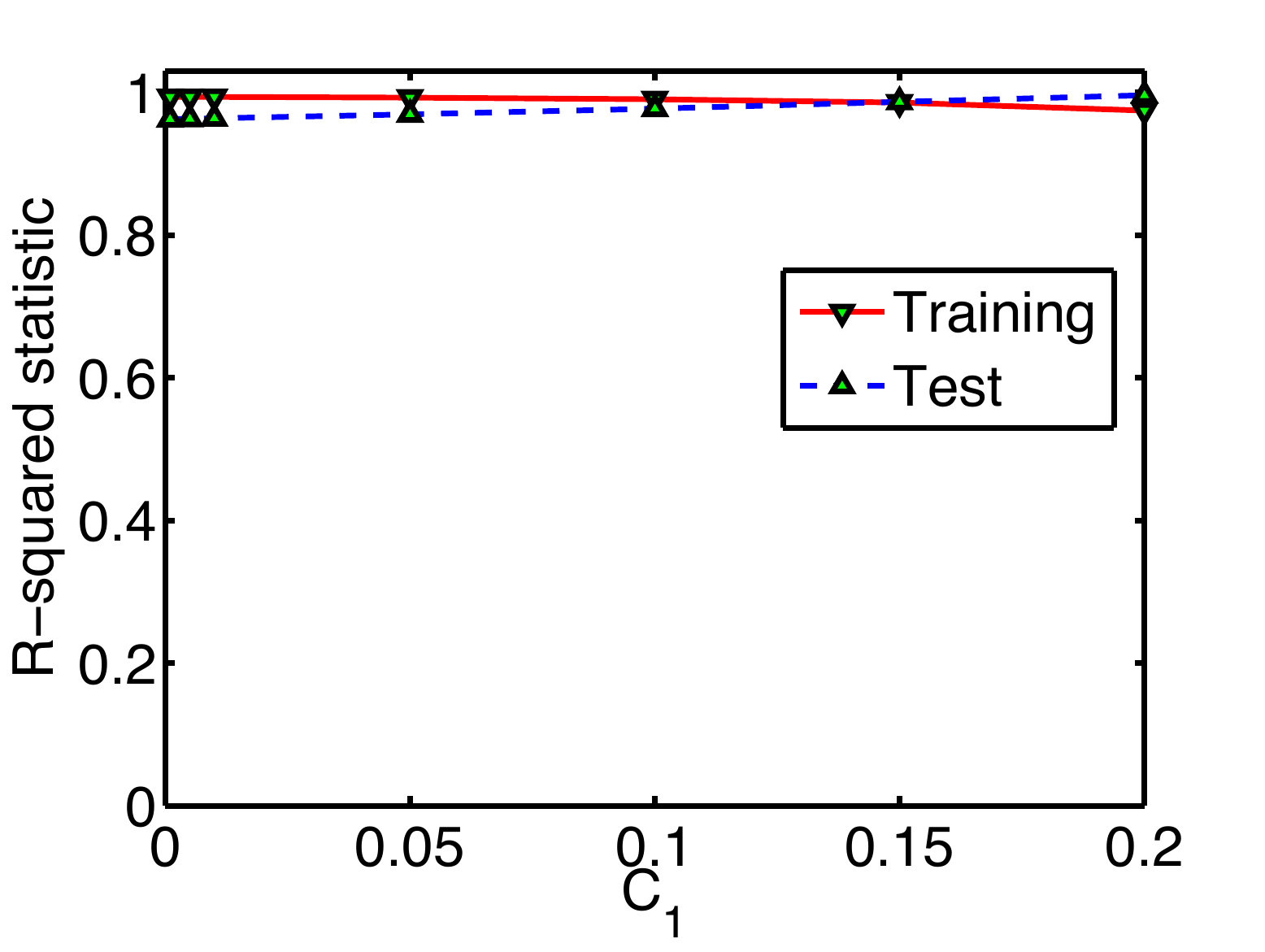}
\caption{\textit{Left:} Operational cost vs $C_{1}$. \textit{Center:} Penalized training loss vs $C_{1}$. \textit{Right:} R-squared statistic. $C_{1} = 0$ corresponds to the baseline, which is the sequential formulation. \label{fig:problem_instance1_results}}
\end{figure}

\vspace*{5pt}
\noindent \textbf{Connection to learning theory:} In the experiment, we used tradeoff parameter $C_1$ to provide a soft constraint. Considering instead the corresponding hard constraint $ \min_{\pi}(\pi_5-\pi_1)\leq\alpha$, the total time must be at least the time for any of the three paths in Figure \ref{fig:problem_instance1}, and thus at least the average of them:
\begin{align}
\nonumber \alpha\geq & \min_{\pi\in\Pi\{\beta,\{\tilde{x}_{i}\}_{i}\}}\pi_5-\pi_1\\
\nonumber \geq & \max\{(\tilde{x}_6+\tilde{x}_{10})^T\beta,(\tilde{x}_6+\tilde{x}_{8}+\tilde{x}_{11})^T\beta,(\tilde{x}_7+\tilde{x}_{9}+\tilde{x}_{11})^T\beta\}\\
\label{eqn:linear-const-relaxation1} \geq & z^T\beta
\end{align}
where 
$$z= \frac{1}{3}[( \tilde{x}_{6} + \tilde{x}_{10}) + ( \tilde{x}_{6} + \tilde{x}_{8} + \tilde{x}_{11}) + ( \tilde{x}_{7} + \tilde{x}_{9} + \tilde{x}_{11})].$$
The main result in Section \ref{sec:bound}, Theorem \ref{thm:main-bound}, is a learning theoretic guarantee in the presence of this kind of arbitrary linear constraint, $z^T\beta\leq \alpha$.

%%Housing and {0-1} Knapsack
\subsection{Housing Prices and the Knapsack Problem} A developer will purchase 3 properties amongst the 6 that are currently for sale and in addition, will remodel them. 
She wants to maximize the total value of the houses she picks (the value of a property is its purchase cost plus the fixed remodeling cost).
 The fixed remodeling costs for the 6 properties are denoted $\{c_{i}\}_{i=1}^{6}$. She estimates the purchase cost of each property from data regarding historical sales, in this case, from the \textit{Boston Housing} data set \citep{uci}, which has 13 features. Let policy $\pi\in\{0,1\}^6$ be the $6$-dimensional binary vector that indicates the properties she purchases. Also, $x_{i}$ represents the features of property $i$ in the training data and $\tilde{x}_{i}$ represents the features of a different property that is currently on sale. The training loss is chosen to be the sum of squares error between the estimated prices $\beta^Tx_i$ and the true house prices $y_i$ for historical sales. The cost (in this case, total value) is the sum of the three property values plus the costs for repair work. A pessimistic bias on total value is chosen to motivate a min-max formulation. The resulting (mixed-integer) program for Step 1 of the simultaneous process is:
\begin{eqnarray}
\lefteqn{\nonumber
\min_{\beta\in\{\beta:\beta\in\R^{13},\|\beta\|_{2}^{2}\leq C_{2}^{*}\}}\sum_{i=1}^{n}(y_{i}-\beta^{T}x_{i})^{2}}
\\
&& + C_{1}\Bigg[\max_{\pi \in \{0,1\}^6}\sum_{i=1}^{6}(\beta^{T}\tilde{x}_{i} + c_{i})\pi_{i} \;\;\; \textrm{ subject to }\;\;\; \sum_{i=1}^{6} \pi_{i} \leq 3 \Bigg]. 
\label{eqn:problem_instance2}
\end{eqnarray}

Notice that the second term above is a $1$-dimensional $\{0,1\}$ knapsack instance. Since the set of policies $\Pi$ does not depend on $\beta$, we can rewrite (\ref{eqn:problem_instance2}) in a cleaner way that was not possible directly with (\ref{eqn:problem_instance1}):
\begin{align}
\nonumber \min_{\beta}\max_{\pi} & \Bigg[\sum_{i=1}^{n}(y_{i}-\beta^{T}x_{i})^{2} + C_{1}\sum_{i=1}^{6}(\beta^{T}\tilde{x}_{i} + c_{i})\pi_{i}\Bigg]\\
\nonumber \textrm{ subject to }&\;\;\; \\
\nonumber \beta & \in \{\beta:\beta\in\R^{13},\|\beta\|_{2}^{2}\leq C_{2}^{*}\}\\
\pi &\in \left\{\pi:  \pi \in \{0,1\}^6, \sum_{i=1}^{6} \pi_{i} \leq 3\right\}. 
\label{eqn:problem_instance2a}
\end{align}

To solve (\ref{eqn:problem_instance2a}) with user-defined parameters $C_{1}$ and $C_{2}$, we use {fminimax}, available through Matlab's Optimization toolbox.\footnote{Version 5.1, Matlab R2010b, Mathworks, Inc.} 

For the training and unlabeled set we chose, there is a change in policy above and below $C_1=0.05$, where different properties are purchased.  Figure \ref{fig:problem_instance2_results}  shows the operational cost which is the predicted total value of the houses after remodeling, the training loss, and $r^2$ values for a range of $C_1$. The training loss and $r^2$ values change by less than $\sim$3.5\%, whereas the total value changes about $6.5$\%.
\textcolor{Black}{ %COLOR SECOND LAYER
We can again draw conclusions in terms of the questions in the introduction as follows.}
\textcolor{Black}{ %COLOR FIRST LAYER
The pessimistic bias shows that even if the developer chose the best response policy to the prices, she might end up with the expected total value of the purchased properties on the order of $6.5$\% less if she is unlucky. Also, we can now produce a realistic model where the total value is $6.5$\% less. We can use this model to help her understand the uncertainty involved in her investment. }%COLOR
%\textcolor{Black}{ %COLOR SECOND LAYER
%Also, if she has a prior belief on the operational profit, she can potentially arrive at the right set of models using fewer examples.
%}

\begin{figure}
\centering 
     \includegraphics[width=.33\textwidth]{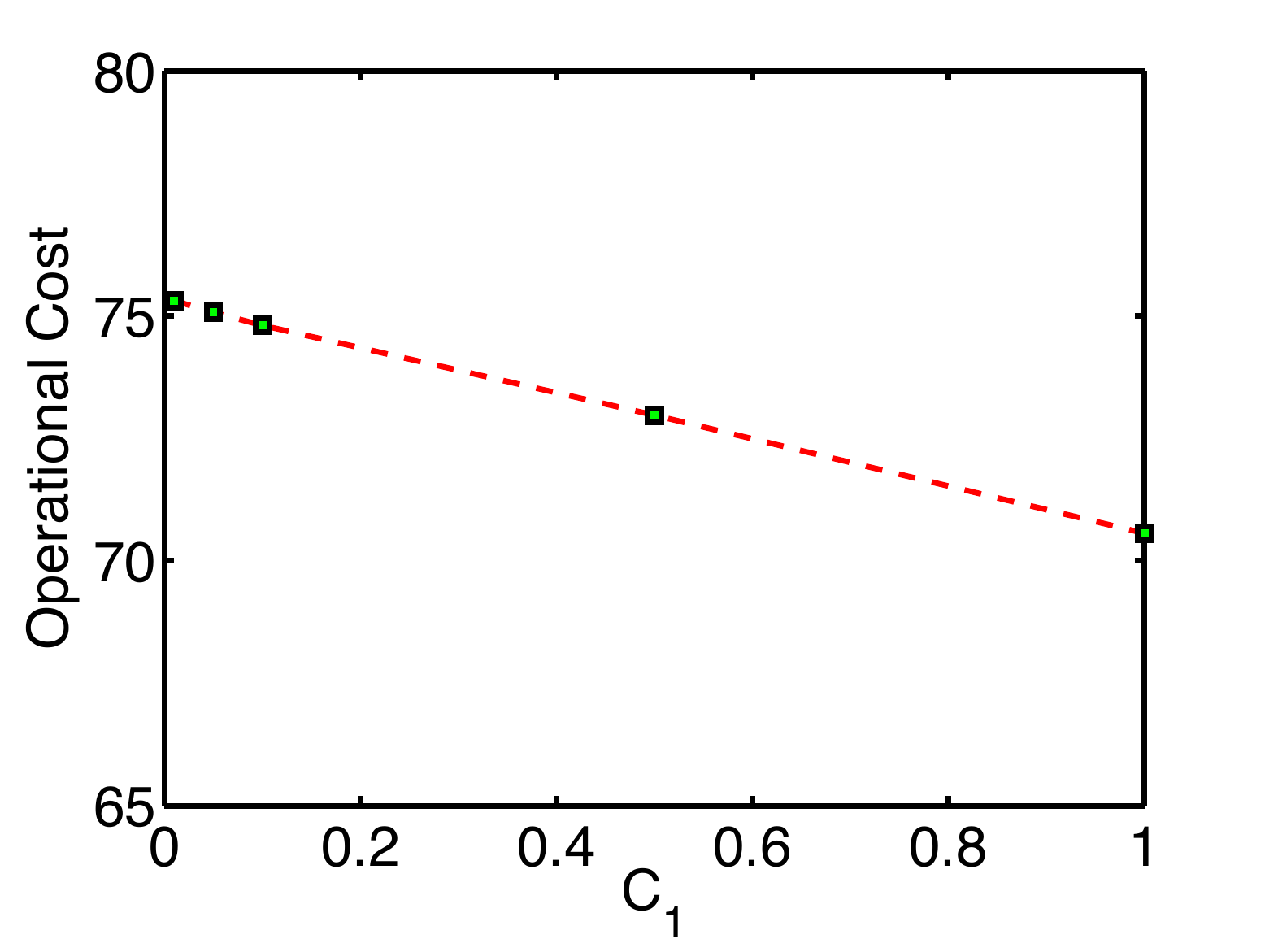}
\hspace{-15pt}
	\includegraphics[width=.33\textwidth]{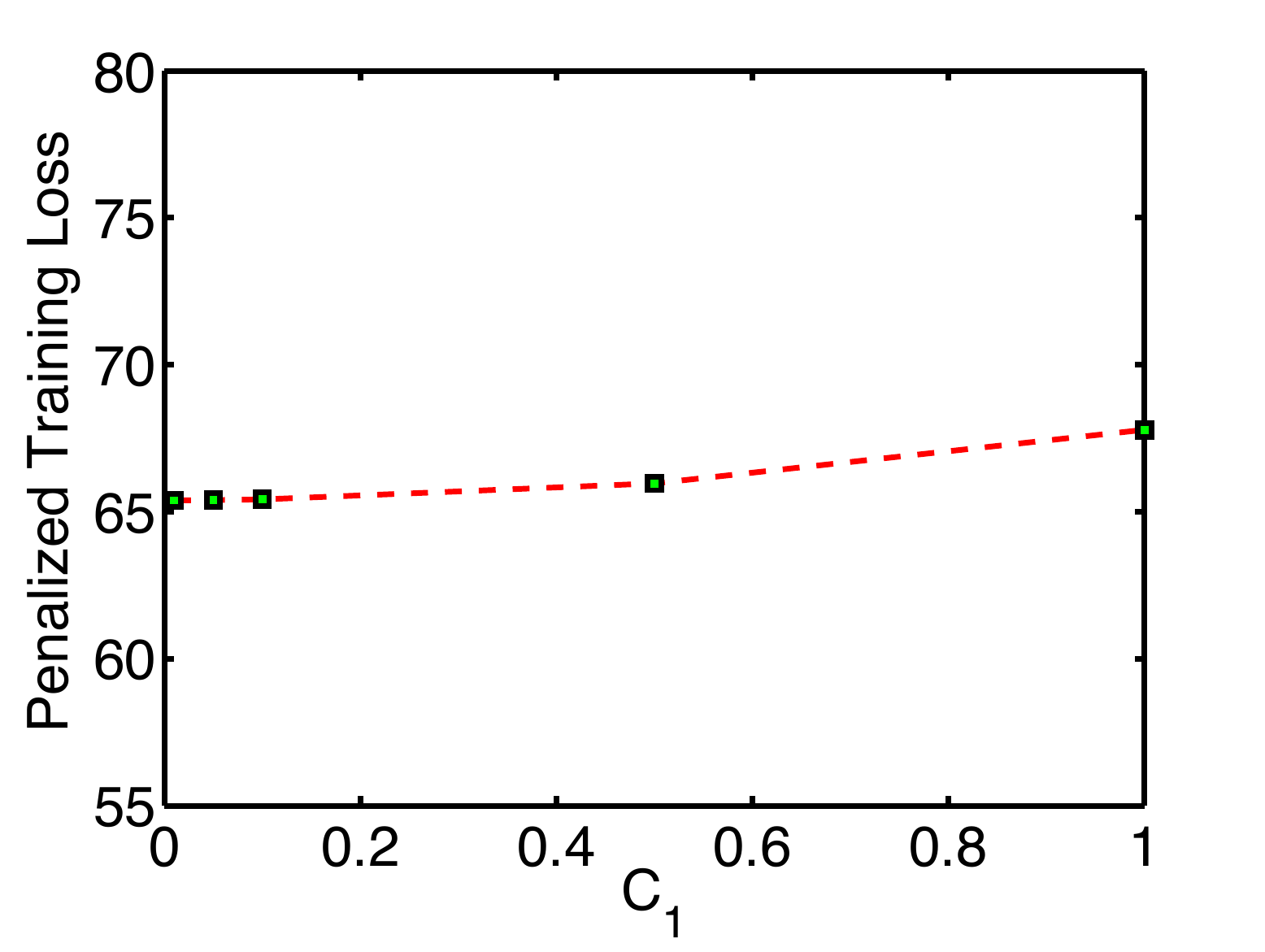}
\hspace{-15pt}
	\includegraphics[width=.33\textwidth]{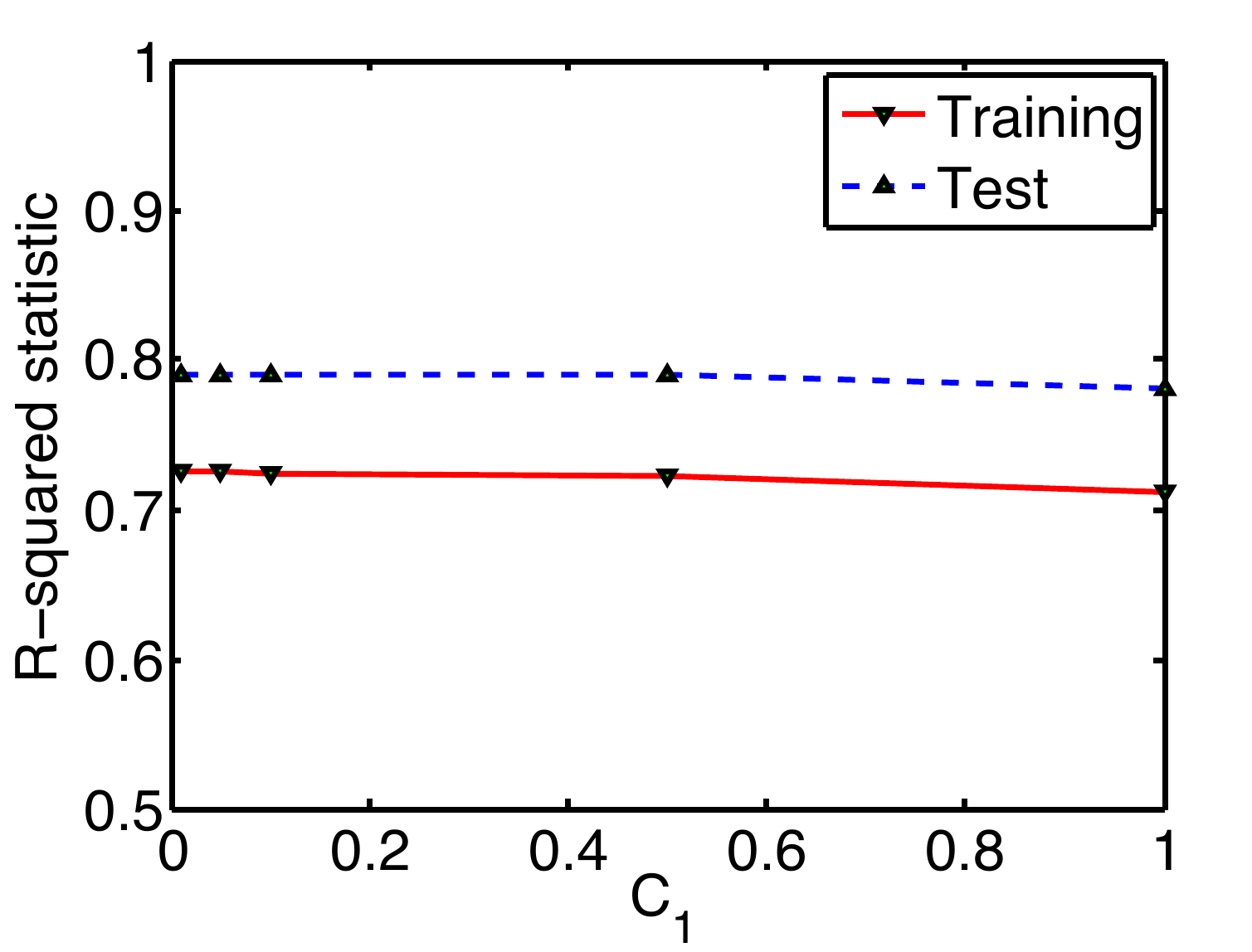}
\caption{\textit{Left:} Operational cost (total value) vs $C_{1}$. \textit{Center:} Penalized training loss vs $C_{1}$. \textit{Right:} R-squared statistic. $C_{1} = 0$ corresponds to the baseline, which is the sequential formulation. \label{fig:problem_instance2_results}}
\end{figure}

Before moving to the next application of the proposed framework, we provide a bound analogous to that of (\ref{eqn:linear-const-relaxation1}).
Let us replace the soft constraint represented by the second term of (\ref{eqn:problem_instance2}) with a hard constraint and then obtain a lower bound:
\begin{align}
\alpha\geq \max_{\pi\in\{0,1\}^6, \sum_{i=1}^{6}\pi_{i} \leq 3} \sum_{i=1}^{6}(\beta^{T}\tilde{x}_{i})\pi_{i}\geq \sum_{i=1}^{6}(\beta^{T}\tilde{x}_{i})\pi'_{i}, 
\label{eqn:linear-const-relaxation2}
\end{align}
where $\pi'$ is some feasible solution of the linear programming relaxation of this problem that also gives a lower objective value. For instance picking $\pi'_i=0.5$ for $i=1,\ldots,6$ is a valid lower bound giving us a looser constraint. 
The constraint can be rewritten:
\[\beta^{T}\left(\frac{1}{2}\sum_{i=1}^{n}\tilde{x}_{i}\right) \leq \alpha.\]  
This is again a linear constraint on the function class parametrized by $\beta$, which we can use for the analysis in Section \ref{sec:bound}.

Note that if all six properties were being purchased by the developer instead of three, the knapsack problem would have a trivial solution and the regularization term would be explicit (rather than implicit).
 
%%Call Center and Scheduling
\subsection{A Call Center's Workload Estimation and Staff Scheduling} 

A call center management wants to come up with the per-half-hour schedule for the staff for a given day between 10am to 10pm. The staff on duty should be enough to meet the demand based on call arrival estimates $N(i), i=1,...,24$. The staff required will depend linearly on the demand per half-hour. The demand per half-hour in turn will be computed based on the Erlang C  model \citep{aldor2009workload} which is also known as the square-root staffing rule. This particular model relates the demand $D(i)$ to the call arrival rate $N(i)$ in the following manner: $D(i) \propto N(i) + c\sqrt{N(i)}$ where $c$ determines where on the QED (Quality Efficiency Driven) curve the center wants to operate on. We make the simplifying assumptions that the service time for each customer is constant, and that the coefficient $c$ is $0$. 

%\begin{wrapfigure}{l}{210pt} 
\begin{figure}
\begin{center}
  \resizebox{200pt}{200pt}{
	\includegraphics[width=0.4\textwidth]{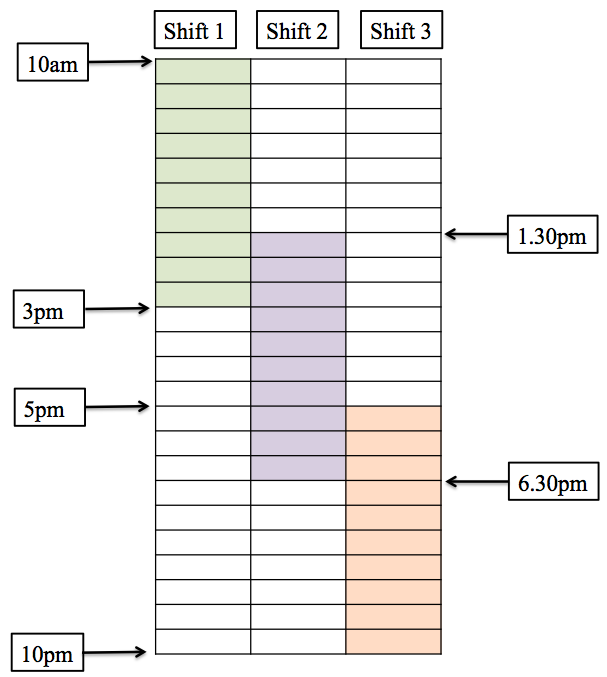}
	}
\caption{The three shifts for the call center. The cells represent half-hour periods, and there are 24 periods per work day. Work starts at 10am and ends at 10pm. \label{fig:problem_instance3}}
\end{center}
\end{figure}
%\end{wrapfigure}

If we know the call arrival rate $N(i)$, we can calculate the staffing requirements during each half hour. If we do not know the call arrival rate, we can estimate it from past data, and make optimistic or pessimistic staffing allocations.

There are additional staffing constraints as shown in Figure \ref{fig:problem_instance3}, namely, there are three sets of employees who work at the center such that: the first set can work only from 10am-3pm, the second can work from 1:30pm-6:30pm, and the third set works from 5pm-10pm.
The operational cost is the total number of employees hired to work that day (times a constant, which is the amount each person is paid). 
The objective of the management is to reduce the number of staff on duty but at the same time maintain a certain quality and efficiency.

The call arrivals are modeled as a poisson process \citep{aldor2009workload}. What previous studies \citep{brown2001root} have discovered about this estimation problem is that the square root of the call arrival rate tends to behave as a linear function of several features,  including: day of the week, time of the day, whether it is a holiday/irregular day, and whether it is close to the end of the billing cycle.  

Data for call arrivals and features were collected over a period of 10 months from Mid-February 2004 to the end of December 2004 \citep[this is the same dataset as in][]{aldor2009workload}. After converting categorical variables into binary encodings  (e.g., each of the 7 weekdays into 6 binary features) the number of features is 36, and we randomly split the data into a training set and test set (2764 instances for training; another 3308 for test). 

We now formalize the optimization problem for the simultaneous process. Let policy $\pi\in \mathbb{Z}_{+}^3$ be a size three vector indicating the number of employees for each of the three shifts. The training loss is the sum of squares error between the estimated square root of the arrival rate $\beta^Tx_i$ and the actual square root of the arrival rate $y_i := \sqrt{N(i)}$. The cost  is proportional to the total number of employees signed up to work, $\sum_i \pi_i$. An optimistic bias on cost is chosen, so that the (mixed-integer) program for Step 1 is:
\begin{eqnarray}
\nonumber
\lefteqn{\min_{\beta:\|\beta\|_{2}^{2}\leq C_{2}^{*}}\sum_{i=1}^{n}(y_{i}-\beta^{T}x_{i})^{2}}\\ 
&&+ C_{1}\Bigg[ \min_{\pi}\sum_{i=1}^{3}\pi_{i} \textrm{ subject to } a_{i}^{T} \pi \geq (\beta^{T}\tilde{x}_{i})^{2} \textrm{ for } i=1,...,24, \pi \in \mathbb{Z}_{+}^3 \Bigg], 
\label{eqn:problem_instance3}
\end{eqnarray}
where Figure \ref{fig:problem_instance3} illustrates the matrix $A$ with the shaded cells containing entry $1$ and $0$ elsewhere. The notation $a_i$ indicates the $i^{\textrm{th}}$ row of $A$:
\begin{align*}
a_{i}(j) = \left\{ \begin{array}{cc} 1 & \textrm{ if staff $j$ can work in half-hour period $i$} \\  0 & \textrm{ otherwise.} \end{array} \right.
\end{align*}

To solve (\ref{eqn:problem_instance3})  we first relax the $\ell_{2}$-norm constraint on $\beta$ by adding another term to the function evaluation, namely $C_2\|\beta\|_{2}^{2}$. This, way we can use a function-evaluation based scheme that works for unconstrained optimization problems. As in the manpower scheduling example, we used an implementation of the Nelder-Mead algorithm, where at each step, Gurobi was used to solve the mixed-integer subproblem for finding the policy.

\textcolor{Black}{ %COLOR FIRST LAYER
Figure \ref{fig:problem_instance3_results} shows the operational cost, the training loss, and $r^2$ values for a range of $C_1$. The training loss and $r^2$ values change only $\sim$1.6\% and $\sim$3.9\% respectively, whereas the operational cost  changes about 9.2\%. 
Similar to the previous two examples, we can again draw conclusions in terms of the questions in Section \ref{sec:intro} as follows.
The optimistic bias shows that the management might incur operational costs on the order of $9$\% less if they are lucky. Further, the simultaneous process produces a reasonable model  where costs are about $9$\% less. If the management team believes they will be reasonably lucky, they can justify designating substantially less than the amount suggested by the traditional sequential process. 
%\textcolor{Black}{ %COLOR SECOND LAYER
%Learning the model itself needs lesser examples if we have sufficient prior on how the operational costs look like.
%}
}%COLOR

\begin{figure}
\centering 
     \includegraphics[width=.33\textwidth]{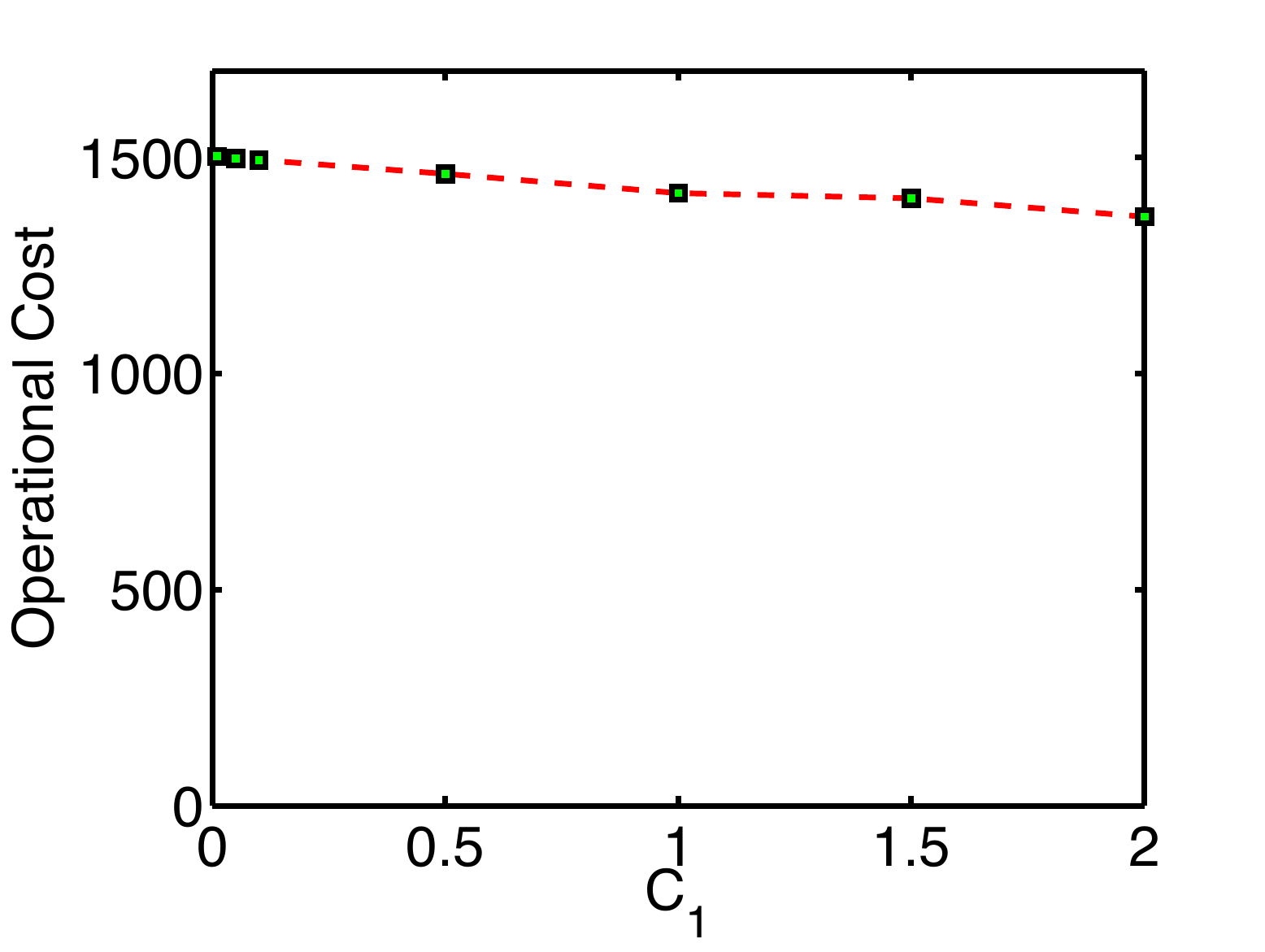}
\hspace{-15pt}
	\includegraphics[width=.33\textwidth]{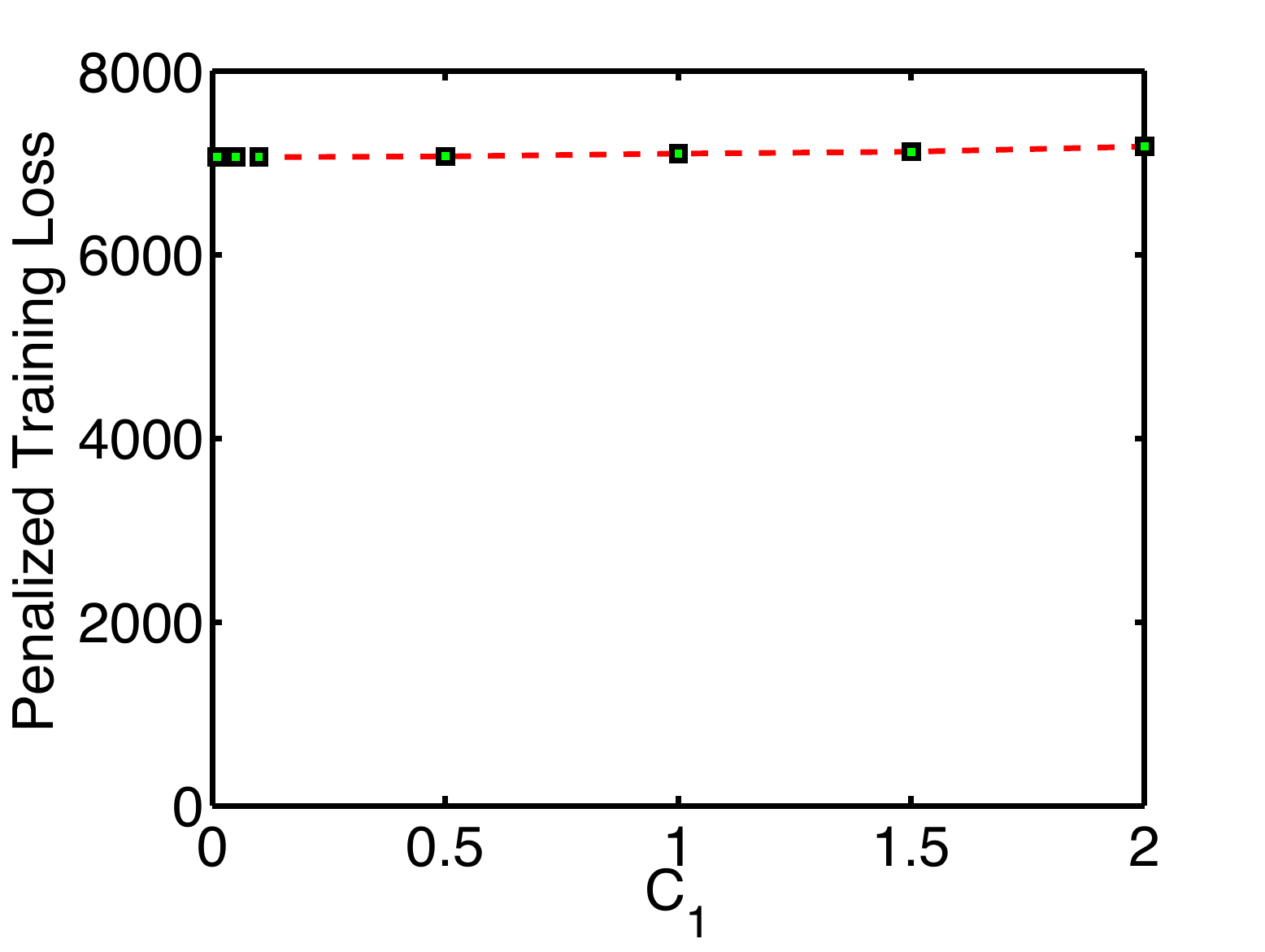}
\hspace{-15pt}
	\includegraphics[width=.33\textwidth]{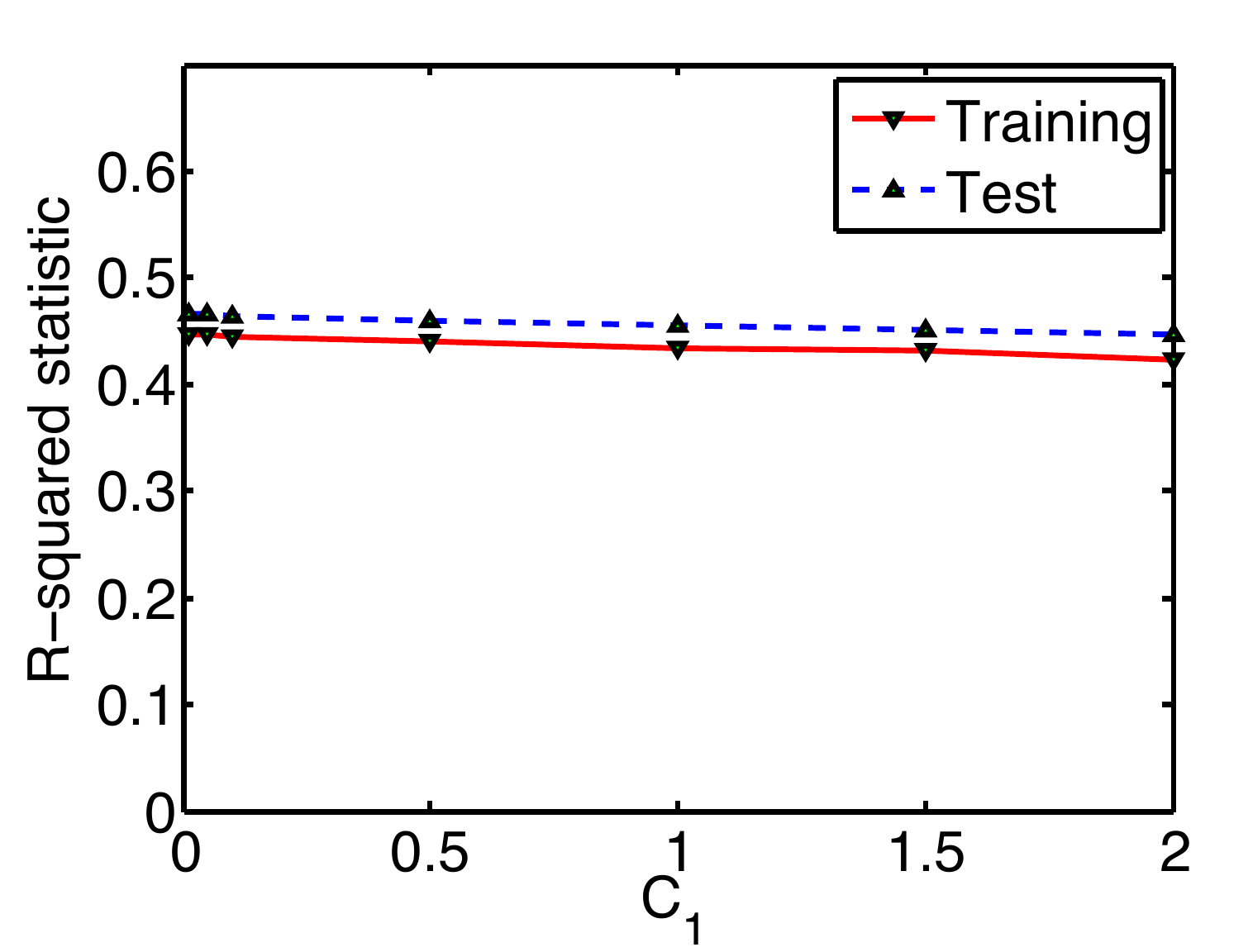}
\caption{\textit{Left:} Operational cost vs $C_{1}$. \textit{Center:} Penalized training loss vs $C_{1}$. \textit{Right:} R-squared statistic. $C_{1} = 0$ corresponds to the baseline, which is the sequential formulation.\label{fig:problem_instance3_results}}
\end{figure}

Let us now investigate the structure of the operational cost regularization term we have in  (\ref{eqn:problem_instance3}). For convenience, let us stack the quantities $(\beta^{T}\tilde{x}_{i})^{2}$ as a vector $b \in \mathbb{R}^{24}$. Also let boldface symbol $\mathbf{1}$ represent a vector of all ones. If we replace the soft constraint represented by the second term with a hard constraint having an upper bound $\alpha$, we get:
\begin{align*}
\alpha & \geq \min_{\pi\in\mathbb{Z}_{+}^3; A\pi \geq b} \sum_{i=1}^{3}\mathbf{1}^{T}\pi \overset{(\dagger)}{\geq}  \min_{\pi\in\mathbb{R}_{+}^3; A\pi \geq b} \sum_{i=1}^{3}\mathbf{1}^{T}\pi \overset{(\ddagger)}= \max_{w\in \mathbb{R}_{+}^{24}; A^{T}w \leq \mathbf{1}}\sum_{i=1}^{24}w_{i}(\beta^{T}\tilde{x}_{i})^{2}\\
& \overset{(*)}{\geq}  \sum_{i=1}^{24}\frac{1}{10}(\beta^{T}\tilde{x}_{i})^{2}.
\end{align*}
Here $\alpha$ is related to the choice of $C_{1}$ and is fixed. $(\dagger)$ represents an LP relaxation of the integer program with $\pi$ now belonging to the positive orthant rather than the cartesian product of set of positive integers. $(\ddagger)$ is due to LP strong duality and $(*)$ is by choosing an appropriate feasible dual variable. Specifically, we pick $w_{i} = \frac{1}{10}$ for $i=1,\ldots,24$, which is feasible because staff cannot work more than 10 half hour shifts (or 5 hours).  With the three inequalities, 
we now have a constraint on $\beta$ of the form:
\[
\sum_{i=1}^{24}(\beta^{T}\tilde{x}_{i})^{2} \leq 10\alpha.
\]
This is a quadratic form in $\beta$ and gives an ellipsoidal feasible set. We already had a simple ellipsoidal feasibility constraint while defining the minimization problem of (\ref{eqn:problem_instance3}) of the form $\|\beta\|_{2}^{2} \leq C_{2}^{*}$. Thus, we can see that our effective hypothesis set (the set of linear functionals satisfying these constraints) has become smaller. This in turn affects generalization. We are investigating generalization bounds for this type of hypothesis set in separate ongoing work.

%%%%%%%%%%%%%%%% ML&TRP
\subsection{The Machine Learning and Traveling Repairman Problem (ML\&TRP) \citep{TuRuJaadt11}}
%Removed citation: TuRuJa11ArXiv,
Recently, power companies have been investing in intelligent ``proactive" maintenance for the power grid, in order to enhance public safety and reliability of electrical service. For instance, New York City has implemented new inspection and repair programs for manholes, where a manhole is an access point to the underground electrical system. Electrical grids can be extremely large (there are on the order of 23,000-53,000 manholes in each borough of NYC), and parts of the underground  distribution network in many cities can be as old as 130 years, dating from the time of Thomas Edison. Because of the difficulties in collecting and analyzing historical electrical grid data, electrical grid repair and maintenance has been performed reactively (fix it only when it breaks), until recently \citep{Urbina04}. These new proactive maintenance programs open the door for machine learning to assist with smart grid maintenance. 

Machine learning models have started to be used for proactive maintenance in NYC, where supervised ranking algorithms are used to rank the manholes in order of predicted susceptibility to failure (fires, explosions, smoke)  so that the most vulnerable manholes can be prioritized \citep{RudinEtAl10,RudinEtAl12,RudinEtAl2011ComputerMagazine}. The machine learning algorithms make reasonably accurate predictions of manhole vulnerability; however, they do not (nor would they, using any other prediction-only technique) take the cost of repairs into account when making the ranked lists.  They do not know that it is unreasonable, for example, if a repair crew has to travel across the city and back again for each manhole inspection, losing important time in the process. The power company must solve an optimization problem to determine the best repair route, based on the machine learning model's output. We might wish to find a policy that is not only supported by the historical power grid data (that ranks more vulnerable manholes above less vulnerable ones), but also would give a better route for the repair crew. An algorithm that could find such a route would lead to an improvement in repair operations on NYC's power grid, other power grids across the world, and improvements in many different kinds of routing operations (delivery trucks, trains, airplanes). 

The simultaneous process could be used to solve this problem, where the operational cost is the price to route the repair crew along a graph, and the probabilities of failure at each node in the graph must be estimated. We call this the  ``the machine learning and traveling repairman problem" (ML\&TRP) and in 
our ongoing work 
%\citep{TuRuJa11ArXiv}
\citep{TuRuJaadt11}
, we have developed several formulations for the ML\&TRP. We demonstrated, using manholes from the Bronx region of NYC, that it is possible to obtain a much more practical route using the ML\&TRP, by taking the cost of the route optimistically into account in the machine learning model.
%\textcolor{Blue}{
%(addressing question Q2 of Section \ref{sec:intro}).
%} 
We showed also that from the routing problem, we can obtain a linear constraint on the hypothesis space, in order to apply the generalization analysis of Section \ref{sec:bound}
%\textcolor{Blue}{
(and in order to address question Q3 of Section \ref{sec:intro}).
%}

%ISAIM text:
%In the longer version of our work we considered other applications, one where the regularization term is the solution of a different type of scheduling problem, and another where it is the solution of a routing problem (the ML\&TRP). In all cases, the simultaneous process performs about equally well for prediction as the sequential process, but gives a range of operational costs.

%%%%%%%%%%%%%%%%%%%%%%%%Section: Connections to Robust Optimization
\section{Connections to Robust Optimization}\label{sec:robust}

\textcolor{Black}{ %COLOR FIRST LAYER
The goal of robust optimization (RO) is to provide the best possible policy that is acceptable under a wide range of situations.\footnote{http://en.wikipedia.org/wiki/Robust\_optimization}
This is different from the simultaneous process, which aims to find the best policies and costs for specific situations. Note that it is not always desirable to have a policy that is robust to a wide range of situations; this is a question of whether to %design the policy around the worst case, or the most probable case. 
respond to every situation simultaneously or whether to understand the single worst situation that could reasonably occur (which is what the pessimistic simultaneous formulation handles). In general, robust optimization can be overly pessimistic, requiring us to allocate enough to handle all reasonable situations; it can be substantially more pessimistic than the pessimistic simultaneous process.
%Depending on the application, it may be better to choose a best response policy to the outcome that is most likely to actually occur than the one that is aimed generally at all reasonable cases, including the worst case. 
% which only requires us to allocate as much as the single worst realistic situation. 
}%COLOR

%(\textit{e.g.}, see \citep{BertsimasBrCa11}). 
%\citep[\textit{e.g.}, see][]{}. 
\textcolor{Black}{ %COLOR FIRST LAYER
In robust optimization, if there are several real-valued parameters involved in the optimization problem, we might declare a reasonable range, called the ``uncertainty set," for each parameter (\textit{e.g.} $a_1 \in [9,10]$, $a_2 \in [1,2]$). Using techniques of RO, we would minimize the largest possible operational cost that could arise from parameter settings in these ranges. Estimation is not usually involved in the study of robust optimization \citep[with some exceptions, see][who consider support vector machines]{Xu2009}. On the other hand, one could choose the uncertainty set according to a statistical model, which is how we will build a connection to RO. Here, we choose the uncertainty set to be the class of models that fit the data to within $\epsilon$, according to some fitting criteria.  
}%COLOR

The major goals of the field of RO include algorithms, geometry, and tractability %\citep{bertsimas04,BertsimasBrCa11}
in finding the best policy, whereas our work is not concerned with finding a robust policy, but we are concerned with estimation, taking the policy into account. 
Tractability for us is not always a main concern as we need to be able to solve the optimization problem, even to use the sequential process. Using even a small optimization problem as the operational cost might have a large impact on the model and decision. If the unlabeled set is not too large, or if the policy optimization problem can be broken into smaller subproblems, there is no problem with tractability. 
An example where the policy optimization might be broken into smaller subproblems is when the policy involves routing several different vehicles, where each vehicle must visit part of the unlabeled set; in that case there is a small subproblem for each vehicle.
On the other hand, even though the goals of the simultaneous process and RO are entirely different, there is a strong connection with respect to the formulations for the simultaneous process and RO, and a class of problems for which they are equivalent. We will explore this connection in this section.

There are other methods that consider uncertainty in optimization, though not via the lens of estimation and learning. In the simplest case, one can perform both local and global sensitivity analysis for linear programs to ascertain uncertainty in the optimal solution and objective, but these techniques generally only handle simple forms of uncertainty \citep{Vanderbei}. 
Our work is also related to stochastic programming, where the goal is to find a policy that is robust to almost all of the possible circumstances (rather than all of them), where there are random variables governing the parameters of the problem, with known distributions \citep{birge1997introduction}. Again, our goal is not to find a policy that is necessarily robust to (almost all of) the worst cases, and estimation is again not the primary concern for stochastic programming, rather it is how to take known randomness into account when determining the policy.

\subsection{Equivalence Between RO and the Simultaneous Process in Some Cases}

In this subsection we will formally introduce RO. In order to connect RO to estimation, we will define the uncertainty set for RO, denoted $\mathcal{F}_{good}$, to be models for which the average loss on the sample is within $\epsilon$ of the lowest possible. Then we will present the equivalence relationship between RO and the simultaneous process, using a minimax theorem.

In Section \ref{sec:formulation}, we had introduced the notation $\{(x_{i},y_{i})\}_i$ and $\{\tilde{x}_{i}\}_{i}$ for labeled and unlabeled data respectively. We had also introduced the class $\mathcal{F}^{unc}$ in which we were searching for a function $f^{*}$ by minimizing an objective of the form (\ref{eqn:reg-train-loss}). 
The uncertainty set $\mathcal{F}_{good}$ will turn out to be a subset of $\mathcal{F}^{unc}$ that depends on $\{(x_{i},y_{i})\}_{i}$ and $f^{*}$ but not on $\{\tilde{x}_{i}\}_{i}$. 

We start with plain (non-robust) optimization, using a general version of the vanilla sequential process. Let $f$ denote an element of the set $\mathcal{F}_{good}$, where $f$ is pre-determined, known and fixed. Let the optimization problem for the policy decision $\pi$ be defined by:
\begin{equation}
\min_{\pi \in \Pi(f;\{\tilde{x}\}_{i})} \Obj(\pi,f;\{\tilde{x}_{i}\}), \hspace*{50pt}\textit{(Base problem)} 
\label{eqn:ro_baseproblem}
\end{equation}
where $\Pi(f;\{\tilde{x}_{i}\})$ is the feasible set for the optimization problem. 
Note that this is a more general version of the sequential process than in Section \ref{sec:formulation}, since we have allowed the constraint set $\Pi$ to be a function of both $f$ and $\{\tilde{x}_{i}\}_i$, whereas in (\ref{eqn:optimisticbias}) and (\ref{eqn:pessimisticbias}), only the objective and not the constraint set can depend on $f$ and $\{\tilde{x}_{i}\}_{i}$. Allowing this more general version of $\Pi$ will allow us to relate (\ref{eqn:ro_baseproblem}) to RO more clearly, and will help us to specify the additional assumptions we need in order to show the equivalence relationship. Specifically, in Section \ref{sec:formulation}, $\Obj$ depends on $(f, \{\tilde{x}_{i}\}_{i})$ but not $\Pi$; whereas in RO, generally $\Pi$ depends on $(f,\{\tilde{x}_{i}\}_{i})$ but not $\Obj$.
The fact that $\Obj$ does not need to depend on $f$ and $\{\tilde{x}_{i}\}_{i}$ is not a serious issue, since we can generally remove their dependence through auxiliary variables.
 For instance, if the problem is a minimization of the form (\ref{eqn:ro_baseproblem}), we can use an auxiliary variable, say $t$, to obtain an equivalent problem: 
\begin{align*}
							   	\min_{\pi,t} t  & \hspace*{50pt}\textit{(Base problem reformulated)}\\
\textrm{ such that }  \pi \in \Pi(f;\{\tilde{x}_{i}\})  &\\
 							    \Obj(\pi,f;\{\tilde{x}_{i}\}) \leq t&
\end{align*}
where the dependence on $( f,\{\tilde{x}_{i}\}_i  )$ is present only in the (new) feasible set.
Since we had assumed $f$ to be fixed, this is a deterministic optimization problem (convex, mixed-integer, nonlinear, etc.).

Now, consider the case when $f$ is not known exactly but only known to lie in the uncertainty set $\mathcal{F}_{good}$. 
The robust counterpart to (\ref{eqn:ro_baseproblem}) can then be written as:
\begin{equation}
 \min_{\pi \in \cuppi \hspace*{-2pt} \Pi(g;\{\tilde{x}\}_{i}) }\;\;\max_{f \in \mathcal{F}_{good}} \Obj(\pi,f;\{\tilde{x}_{i}\}) \hspace*{50pt}\textit{(Robust counterpart)} 
\label{eqn:ro_counterpart}
\end{equation}
where we obtain a ``robustly feasible solution" that is guaranteed to remain feasible for all values of $f \in \mathcal{F}_{good}$. In general, (\ref{eqn:ro_counterpart}) is much harder to solve than (\ref{eqn:ro_baseproblem}) and is a topic of much interest in the robust optimization community. 
As we discussed earlier, there is no focus in (\ref{eqn:ro_counterpart}) on estimation, but it is possible to embed an estimation problem within the description of the set $\mathcal{F}_{good}$, which we now define formally.

In Section \ref{sec:experiments}, $\Fr$ (a subset of $\mathcal{F}^{unc}$) was defined as the set of linear functionals with the property that $R(f) \leq C_{2}^{*}$. That is,
$$\Fr = \left\{f: f \in \mathcal{F}^{unc}, R(f) \leq C_{2}^{*}\right\}.$$
We define $\mathcal{F}_{good}$ as a subset of $\Fr$ by adding an additional property:
\begin{align}
\label{fgood}
\mathcal{F}_{good} = \left\{f:   f \in \Fr, \sum_{i=1}^{n} l\left(f(x_i),y_i\right) \leq  \sum_{i=1}^{n} l\left(f^{*}(x_i),y_i\right) + \epsilon \right\}, 
\end{align}
for some fixed positive real $\epsilon$. In (\ref{fgood}), again $f^{*}$ is a solution that minimizes the objective in (\ref{eqn:reg-train-loss}) over $\mathcal{F}^{unc}$.  The right hand side of the inequality in (\ref{fgood}) is thus constant, and we will henceforth denote it with a single quantity $C_{1}^{*}$. Substituting this definition of $\mathcal{F}_{good}$ in (\ref{eqn:ro_counterpart}), and further making an important assumption (denoted \textbf{A1}) that $\Pi$ is not a function of $(f,\{\tilde{x}_{i}\}_{i})$, we get the following optimization problem:
\begin{equation}
 \min_{\pi \in \Pi} \max_{\{f \in \Fr: \sum_{i=1}^{n}l(f(x_{i}),y_{i}) \leq C_{1}^{*}\}} \Big[ \Obj\left(\pi,f,\{\tilde{x}_{i}\}_i\right)\Big] \hspace*{10pt}\textit{(Robust counterpart with assumptions)}
 \label{eqn:risk-averse-loss}
\end{equation}
where $C_{1}^{*}$ now controls the amount of the uncertainty via the set $\mathcal{F}_{good}$. 

Before we state the equivalence relationship, we restate the formulations for optimistic and pessimistic biases on operational cost in the simultaneous process from (\ref{eqn:optimisticbias}) and (\ref{eqn:pessimisticbias}):
\begin{align}
%\label{eqn:optimisticbiasfull}
\nonumber
\min_{f\in\mathcal{F}^{unc}} \left[ \sum_{i=1}^{n} l\left(f(x_i),y_i\right) + C_{2}R(f) + C_1 \min_{\pi\in\Pi} \Obj\left(\pi,f,\{\tilde{x}_{i}\}_i\right)  \right]& \textit{(Simultaneous optimistic)}\\
\label{eqn:pessimisticbiasfull}
\min_{f\in\mathcal{F}^{unc}} \left[ \sum_{i=1}^{n} l\left(f(x_i),y_i\right) + C_{2}R(f) - C_1 \min_{\pi\in\Pi} \Obj\left(\pi,f,\{\tilde{x}_{i}\}_i\right)  \right]& \textit{(Simultaneous pessimistic)}
\end{align}

Apart from the assumption \textbf{A1} on the decision set $\Pi$ that we made in (\ref{eqn:risk-averse-loss}), we will also assume that $\mathcal{F}_{good}$ defined in (\ref{fgood}) is convex; this will be assumption \textbf{A2}. If we also assume that the objective $\Obj$ satisfies some nice properties (\textbf{A3}), and that uncertainty is characterized via the set $\mathcal{F}_{good}$, then we can show that the two problems, namely (\ref{eqn:pessimisticbiasfull}) and (\ref{eqn:risk-averse-loss}), are equivalent. Let $\Leftrightarrow$ denote equivalence between two problems, meaning that a solution to one side translates into the solution of the other side for some parameter values ($C_{1},C_{1}^{*},C_{2},C_{2}^{*}$).

\begin{proposition}
Let $\Pi(f;\{\tilde{x}_{i}\}_{i}) = \Pi$ be compact, convex, and independent of parameters $f$ and $\{\tilde{x}_{i}\}_{i}$ (assumption \textbf{A1}). Let $\{f \in \Fr: \sum_{i=1}^{n}l(f(x_{i}),y_{i}) \leq C_{1}^{*}\}$ be convex (assumption \textbf{A2}). Let the cost (to be minimized) $\Obj(\pi,f,\{\tilde{x}_{i}\}_{i})$ be concave continuous in $f$ and convex continuous in $\pi$ (assumption \textbf{A3}). Then, the robust optimization problem (\ref{eqn:risk-averse-loss}) is equivalent to the pessimistic bias optimization problem (\ref{eqn:pessimisticbiasfull}). That is,
\begin{eqnarray*}
 \lefteqn{\min_{\pi \in \Pi} \max_{\{f \in \Fr: \sum_{i=1}^{n}l(f(x_{i}),y_{i}) \leq C_{1}^{*}\}} \Big[ \Obj(\pi,f,\{\tilde{x}_{i}\}_i)\Big]  \Leftrightarrow}\\
 && \min_{f\in\mathcal{F}^{unc}} \left[ \sum_{i=1}^{n} l\left(f(x_i),y_i\right) + C_{2}R(f) - C_1 \min_{\pi\in\Pi} \Obj\left(\pi,f,\{\tilde{x}_{i}\}_i\right)  \right].
\end{eqnarray*}
\label{prop:ro-bias-equivalence}
\end{proposition}
\begin{remark}
That the equivalence applies to linear programs (LPs) is clear because the objective is linear and the feasible set is generally a polyhedron, and is thus convex. For integer programs, the objective $\Obj$ satisfies continuity, but the feasible set is typically not convex, and hence, the result does not generally apply to integer programs. In other words, the requirement that the constraint set $\Pi$ be convex excludes integer programs. 
\end{remark}

To prove Proposition \ref{prop:ro-bias-equivalence}, we restate a well-known generalization of von Neumann's minimax theorem and some related definitions.

\begin{definition}
A linear topological space (also called a topological vector space) is a vector space over a topological field (typically, the real numbers with their standard topology) with a topology such that vector addition and scalar multiplication are continuous functions. For example, any normed vector space is a linear topological space. A function $h$ is upper semicontinuous at a point $p_{0}$ if for every $\epsilon > 0$ there exists a neighborhood $U$ of $p_{0}$ such that $h(p) \leq h(p_{0}) + \epsilon$ for all $p \in U$. A function $h$ defined over a convex set is quasi-concave if for all $p,q$ and $\lambda \in [0,1]$ we have $h(\lambda p + (1-\lambda)q) \geq \min(h(p),h(q))$. Similar definitions follow for lower semicontinuity and quasi-convexity.
\end{definition}

\begin{theorem} \citep[Sion's minimax theorem][]{sion1958general} Let $\Pi$ be a compact convex subset of a linear topological space and $\Xi$ be a convex subset of a linear topological space. Let $G(\pi,\xi)$ be a real function  on $\Pi \times \Xi$ such that
\begin{enumerate}
\item[(i)] $G(\pi,\cdot)$ is upper semicontinuous and quasi-concave on $\Xi$ for each $\pi \in \Pi$;
\item[(ii)] $G(\cdot,\xi)$ is lower semicontinuous and quasi-convex on $\Pi$ for each $\xi \in \Xi$. 
\end{enumerate}
Then
\begin{align*}
\min_{\pi\in \Pi}\sup_{\xi \in \Xi} G (\pi,\xi) = \sup_{\xi \in \Xi}\min_{\pi\in \Pi} G (\pi,\xi)
\end{align*}
\label{theorem:sion}
\end{theorem}
We can now proceed to the proof of Proposition (\ref{prop:ro-bias-equivalence}).
\begin{proof} \textit{(Of Proposition \ref{prop:ro-bias-equivalence})}
We start from the left hand side of the equivalence we want to prove:
\begin{align*}
 \textrm{} & \min_{\pi \in \Pi} \max_{\{f \in \Fr: \sum_{i=1}^{n}l(f(x_{i}),y_{i}) \leq C_{1}^{*}\}} \Big[ \Obj(\pi,f,\{\tilde{x}_{i}\}_i)\Big] \\
\overset{(a)}{\Leftrightarrow} & \max_{\{f \in \Fr: \sum_{i=1}^{n}l(f(x_{i}),y_{i}) \leq C_{1}^{*}\}}  \min_{\pi \in \Pi} \Big[ \Obj(\pi,f,\{\tilde{x}_{i}\}_i)\Big] \\
\overset{(b)}{\Leftrightarrow} &  \max_{f \in \mathcal{F}^{unc}}\Big[  - \frac{1}{C_{1}}\Big( \sum_{i=1}^{n}l(f(x_{i}),y_{i})  - C_{1}^{*}\Big)  -\frac{C_{2}}{C_{1}}\Big( R(f) - C_{2}^{*} \Big) +  \min_{\pi \in \Pi} \Obj(\pi,f,\{\tilde{x}_{i}\}_{i}) \Big] \\
\overset{(c)}{\Leftrightarrow} &  \min_{f\in\mathcal{F}^{unc}} \left[ \sum_{i=1}^{n} l\left(f(x_i),y_i\right) + C_{2}R(f) - C_1 \min_{\pi\in\Pi} \Obj\left(\pi,f,\{\tilde{x}_{i}\}_i\right)  \right].\\
\end{align*}
which is the right hand side of the logical equivalence in the statement of the theorem. In step $(a)$ we applied Sion's minimax theorem (Theorem \ref{theorem:sion}) which is satisfied because of the assumptions we made.
In step $(b)$, we picked Lagrange coefficients, namely $\frac{1}{C_{1}}$ and $\frac{C_{2}}{C_{1}}$, both of which are positive. In particular, $C_{1}^{*}$ and $C_{1}$ as well as $C_{2}^{*}$ and $C_{2}$ are related by the Lagrange relaxation equivalence (strong duality).
In $(c)$, we multiplied the objective with $C_{1}$ throughout, pulled the negative sign in front, and removed the constant terms $C_{1}^{*}$ and $C_{2}C_{2}^{*}$ and used the following observation:  $\max_{a} - g(a) =  -\min_{a} g(a)$; and finally, removed the negative sign in front as this does not affect equivalence. %\qed
\end{proof}

The equivalence relationship of Proposition \ref{prop:ro-bias-equivalence} shows that there is a problem class in which each instance can be viewed either as a RO problem or an estimation problem with an operational cost bias. 
We can use ideas from RO to make the simultaneous process more general. 
Before doing so, we will characterize $\mathcal{F}_{good}$ for several specific loss functions.

\subsection{Creating Uncertainty Sets for RO Using Loss Functions from Machine Learning}
Let us for simplicity specialize our loss function to the least squares loss. Let $X$ be an $n \times p$ matrix with each training instance $x_{i}$ forming the $i^{th}$ row. Also let $Y$ be the $n$-dimensional vector of all the labels $y_{i}$. Then the loss term of (\ref{eqn:reg-train-loss}) can be written as:
$$\sum_{i=1}^{n}(y_{i} - f(x_{i}))^{2} = \sum_{i=1}^{n}(y_{i} - \beta^{T}x_{i})^{2} = \|Y - X\beta\|_{2}^{2}.$$
Let $\beta^{*}$ be a parameter corresponding to $f^{*}$ in (\ref{eqn:reg-train-loss}). Then the definition of $\mathcal{F}_{good}$ in terms of the least squares loss is:
$$\mathcal{F}_{good} = \{f: f\in \Fr, \|Y - X\beta\|_{2}^{2} \leq \|Y - X\beta^{*}\| _{2}^{2} + \epsilon \} = \{f: f\in \Fr, \|Y - X\beta\|_{2}^{2} \leq C_{1}^{*} \}.$$
Since each $f\in\F_{good}$ corresponds to at least one $\beta$, the optimization of (\ref{eqn:reg-train-loss}) can be performed with respect to $\beta$. In particular, the constraint $\|Y-X\beta\| \leq C_{1}^{*}$ is an ellipsoid constraint on $\beta$.
For the purposes of the robust counterpart in (\ref{eqn:ro_counterpart}), we can thus say that the uncertainty is of the ellipsoidal form. 
In fact, ellipsoidal constraints on uncertain parameters are widely used in robust optimization, especially because the resulting optimization problems often remain tractable. 

Box constraints are also a popular way of incorporating uncertainty into robust optimization. For box constraints, the uncertainty over the $p$-dimensional parameter vector $\beta = [\beta_{1},...,\beta_{p}]^{T}$ is written for
$i=1,...,p$ as $LB_{i} \leq \beta_{i} \leq UB_{i},$
where $\{LB_{i}\}_{i}$ and $\{UB_{i}\}_{i}$ are real-valued upper and lower bounds that together define the box intervals.

Our main point in this subsection is that one can potentially derive a very wide range of uncertainty sets for robust optimization using different loss functions from machine learning. Box constraints and ellipsoidal constraints are two simple types of constraints that could potentially be the set $\F_{good}$, which arise from two different loss functions, as we have shown. The least squares loss leads to ellipsoidal constraints on the uncertainty set, but it is unclear what the structure would be for uncertainty sets arising from the 0-1 loss, ramp loss, hinge loss, logistic loss and exponential loss among others. Further, it is possible to create a loss function for fitting data to a probabilistic model using the method of maximum likelihood; uncertainty sets for maximum likelihood could thus be established. Table \ref{table:lossfunc-uncertainty} shows several different popular loss functions and the uncertainty sets they might lead to. Many of these new uncertainty sets do not always give tractable mathematical programs, which could explain why they are not commonly considered in the optimization literature. 

\begin{table}
\begin{center}
\begin{tabular}{|c|c|}
\hline
Loss function & Uncertainty set description\\
\hline
least squares & $\|Y-X\beta\|_{2}^{2} \leq \|Y-X\beta^{*}\|_{2}^{2} + \epsilon$ (ellipsoid) \\
0-1 loss & $\mathbf{1}_{[f(x_{i}) \neq y_{i}]} \leq \mathbf{1}_{[f^{*}(x_{i}) \neq y_{i}]} + \epsilon$\\
logistic loss & $\sum_{i=1}^{n}\log(1+e^{-y_{i}f(x_{i})}) \leq\sum_{i=1}^{n}\log(1+e^{-y_{i}f^{*}(x_{i})}) + \epsilon$\\
exponential loss & $\sum_{i=1}^{n}e^{-y_{i}f(x_{i})} \leq \sum_{i=1}^{n}e^{-y_{i}f^{*}(x_{i})} + \epsilon$\\
ramp loss & $\sum_{i=1}^{n} \min(1,\max(0,1-y_{i}f(x_{i}))) \leq \sum_{i=1}^{n} \min(1,\max(0,1-y_{i}f^{*}(x_{i}))) + \epsilon$ \\
hinge loss & $\sum_{i=1}^{n} \max(0,1-y_{i}f(x_{i})) \leq \sum_{i=1}^{n} \max(0,1-y_{i}f^{*}(x_{i})) + \epsilon$\\
\hline
\end{tabular}
\end{center}
\label{table:lossfunc-uncertainty}
\caption{Table showing a summary of different possible uncertainty set descriptions that are based on ML loss functions.}
\end{table}

\vspace*{5pt}
\noindent \textbf{The sequential process for RO.}
If we design the uncertainty sets as described above, with respect to a machine learning loss function, the sequential process described in Section \ref{sec:formulation} can be used with robust optimization. This proceeds in three steps: 
\begin{enumerate}
\item use a learning algorithm on the training data to get $f^*$, 
\item establish an uncertainty set based on the loss function and $f^*$, for example, ellipsoidal constraints arising from the least squares loss (or one could use any of the new uncertainty sets discussed in the previous paragraph), 
\item use specialized optimization techniques to solve for the best policy, with respect to the uncertainty set.
\end{enumerate}

 We note that the uncertainty sets created by the $0$-$1$ loss and ramp loss for instance, are non-convex, consequently assumption (\textbf{A2}) and Proposition \ref{prop:ro-bias-equivalence} do not hold for robust optimization problems that use these sets.
 
\subsection{The Overlap Between The Simultaneous Process and RO}

On the other end of the spectrum from robust optimization, one can think of ``optimistic" optimization where we are seeking the best value of the objective in the best possible situation (as oppose to the worst possible situation in RO). For optimistic optimization, more uncertainty is favorable, and we find the best policy for the best possible situation. This could be useful in many real applications where one not only wants to know the worst-case conservative policy but also the best case risk-taking policy. 
A typical formulation, following (\ref{eqn:ro_counterpart}) can be written as:
\begin{equation*}
\min_{\pi \in \cuuppi \hspace*{-5pt}\Pi(g;\{\tilde{x}\}_{i}) }\;\;\min_{f \in \mathcal{F}_{good}} \Obj(\pi,f;\{\tilde{x}_{i}\}).\hspace*{30pt}(\textit{Optimistic optimization})
%\label{eqn:oo_counterpart}
\end{equation*}
In optimistic optimization, we view operational cost optimistically ($\min_{f \in \mathcal{F}_{good}} \Obj$) whereas in the robust optimization counterpart (\ref{eqn:ro_counterpart}), we view operational cost conservatively ($\max_{f \in \mathcal{F}_{good}} \Obj$). The policy $\pi^*$ is feasible in more situations in RO ($\min_{\pi \in \cap_{g\in \mathcal{F}_{good}\Pi}}$) since it must be feasible with respect to each $g\in \mathcal{F}_{good}$, whereas the $\Obj$ is lower in optimistic optimization ($\min_{\pi \in \cup_{g\in \mathcal{F}_{good}}\Pi}$) since it need only be feasible with respect to at least one of the $g$'s. Optimistic optimization has not been heavily studied, possibly because a (min-min) formulation is relatively easier to solve than its (min-max) robust counterpart, and so is less computationally interesting. Also, one generally plans for the worst case more often than for the best case, particularly when no estimation is involved. In the case where estimation is involved, both optimistic and robust optimization could potentially be useful to a practitioner. 

Both optimistic optimization and robust optimization, considered with respect to uncertainty sets $\F_{good}$, have non-trivial overlap with the simultaneous process. In particular, we showed in Proposition \ref{prop:ro-bias-equivalence} that 
pessimistic bias on operational cost is equivalent to robust optimization under specific conditions on $\Obj$ and $\Pi$.
Using an analogous proof, one can show that optimistic bias on operational cost is equivalent to optimistic optimization under the same set of conditions. 
Both robust and optimistic optimization and the simultaneous process encompass large classes of problems, some of which overlap. Figure \ref{fig:set-diagram} represents the overlap between the three classes of problems. There is a class of problems that fall into the simultaneous process, but are not equivalent to robust or optimistic optimization problems. These are problems where we use operational cost to assist with estimation, as in the call center example and ML\&TRP discussed in Section~\ref{sec:experiments}. 
Typically problems in this class have $\Pi = \Pi(f;\{\tilde{x}_{i}\}_{i})$. This class includes problems where the bias can be either optimistic or pessimistic, and for which $F_{good}$ has a complicated structure, 
beyond ellipsoidal or box constraints. 
There are also problems contained in either robust optimization or optimistic optimization alone and do not belong to the simultaneous process. Typically, again, this is when $\Pi$ depends on $f$. Note that the housing problem presented in Section \ref{sec:experiments} lies within the intersection of optimistic optimization and the simultaneous process; this can be deduced from (\ref{eqn:problem_instance2a}).

\begin{figure}
\centering
\includegraphics[width=.4\textwidth]{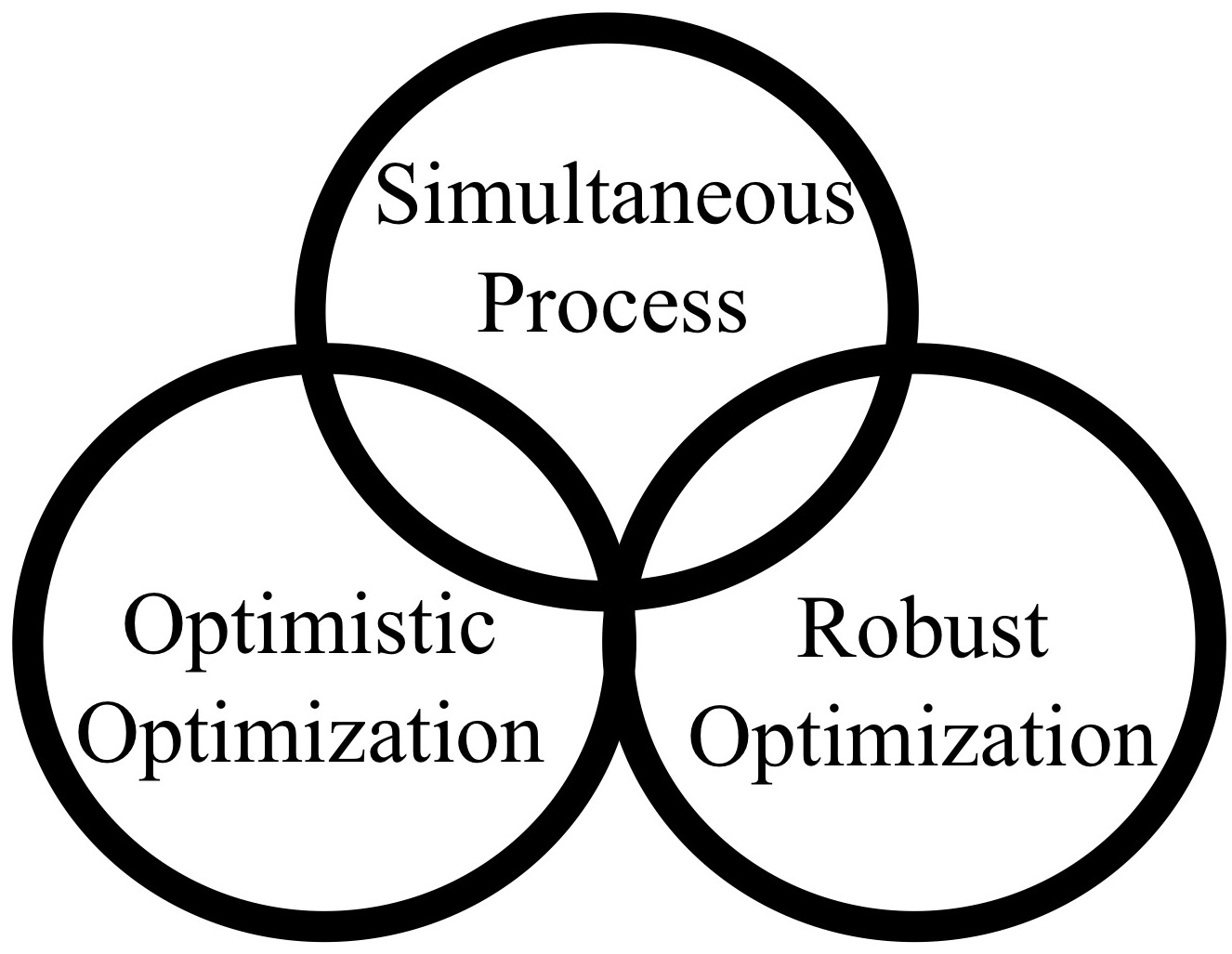}
\caption{
Set based description of the proposed framework (top circle) and its relation to robust (right circle) and optimistic (left circle) optimizations. The regions of intersection are where the conditions on the objective $\Obj$ and the feasible set $\Pi$ are satisfied.
}
\label{fig:set-diagram}
\end{figure}

\textcolor{Black}{ %COLOR FIRST LAYER
In Section \ref{sec:bound}, we will provide statistical guarantees for the simultaneous process. These are very different from the style of probabilistic guarantees in the robust optimization literature.  
There are some ``sample complexity'' bounds in the RO literature of the following form: how many observations of uncertain data are required (and applied as simultaneous constraints) to maintain robustness of the solution with high probability? There is an unfortunate overlap in terminology; these are totally different problems to the sample complexity bounds in statistical learning theory. 
From the learning theory perspective, we ask:  how many training instances does it take to come up with a model $\beta$ that we reasonably know to be good? We will answer that question for a very general class of estimation problems.
}%COLOR

%%%%%%%%%%%%%%%%%Section:Generalization bounds
\section{Generalization Bound with New Linear Constraints} \label{sec:bound}

\textcolor{Black}{ %COLOR FIRST LAYER
In this section, we give statistical learning theoretic results for the simultaneous process that involve counting integer points in convex bodies.
Generalization bounds are probabilistic guarantees, that often depend on some measure of the complexity of the hypothesis space. 
 Limiting the complexity of the hypothesis space equates to a better bound. 
 In this section, we consider the complexity of hypothesis spaces that results from an operational cost bias.}
% \textcolor{Blue}{%COLOR
 This enables us to answer in a quantitative manner, question Q3 in the introduction: ``Can our intuition about how much it will cost to solve a problem help us produce a better probabilistic model?''
% } 

 \textcolor{Black}{ %COLOR FIRST LAYER
Generalization bounds have been well established for \emph{norm-based} constraints on the hypothesis space, but the emphasis has been more on qualitative dependence (e.g., using big-O notation) and the constants are not emphasized. On the other hand, for a practitioner, every prior belief should reduce the number of examples they need to collect, as these examples may each be expensive to obtain; thus constants within the bounds, and even their approximate values, become important \citep{Bousquet03}. We thus provide bounds on the covering number for new types of hypothesis spaces, emphasizing the role of constants.
}%COLOR

To establish the bound, it is sufficient to provide an upper bound on the covering number. There are many existing generic generalization bounds in the literature \citep[e.g.,][]{bartlett02}, which combined with our bound, will yield a specific generalization bound for machine learning with operational costs, as we will construct in Theorem \ref{thm:fullbound}.

In Section \ref{sec:experiments}, we showed that a bias on the operational cost can sometimes be transformed into linear constraints on model parameter $\beta$ (see equations (\ref{eqn:linear-const-relaxation1}) and (\ref{eqn:linear-const-relaxation2})). There is a broad class of other problems for which this is true, for example, for applications related to those presented in Section \ref{sec:experiments}.
%, including the ML\&TRP \citep{TuRuJa11ArXiv}.
Because we are able to obtain linear constraints for such a broad class of problems, we will analyze the case of linear constraints here. The hypothesis we consider is thus the intersection of an $\ell_q$ ball and a halfspace. This is illustrated in Figure \ref{fig:set-diagram2}. 
\begin{figure}
\centering
\includegraphics[width=.5\textwidth]{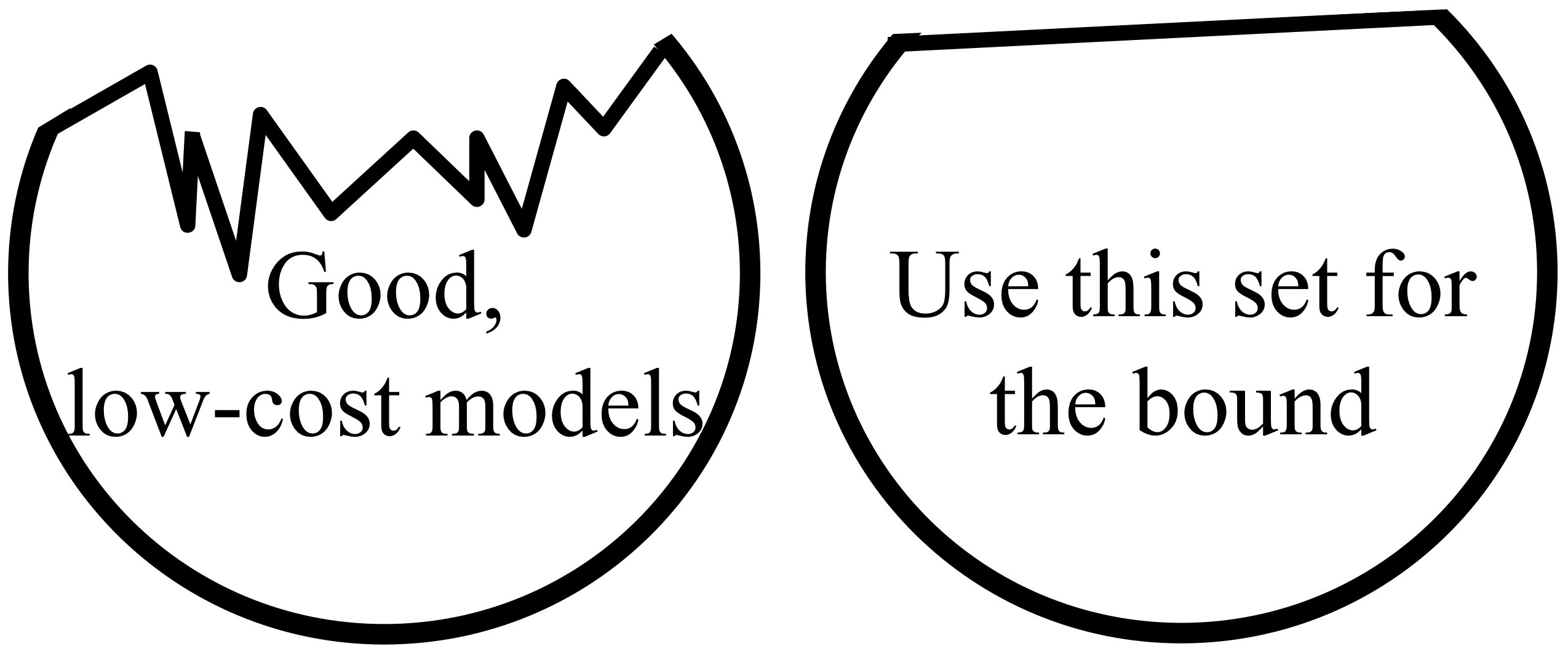}
\caption{
Left: hypothesis space for intersection of good models (circular, to represent $\ell_q$ ball) with low cost models (models below cost threshold, one side of wiggly curve). Right: relaxation to intersection of a half space with an $\ell_q$ ball.
}
\label{fig:set-diagram2}
\end{figure}

\textcolor{Black}{ %COLOR FIRST LAYER
The plan for the rest of the section is as follows. We will introduce the quantities on which our main result in this section depends. Then, we will state the main result (Theorem \ref{thm:main-bound}). Following that, we will build up to a generalization bound (Theorem \ref{thm:fullbound}) that incorporates Theorem \ref{thm:main-bound}. After that will be the proof of Theorem \ref{thm:main-bound}.
}%COLOR

\textcolor{Black}{ %COLOR FIRST LAYER
\begin{definition} \textit{\citep[Covering Number,][]{kol61}} Let $A \subseteq \Gamma$ be an arbitrary set and $(\Gamma, \rho)$ a (pseudo-)metric space. Let $|\cdot|$ denote set size. 
\begin{itemize}
\item For any $\epsilon > 0$, an \textbf{$\epsilon$-cover} for $A$ is a finite set $U \subseteq \Gamma$ (not necessarily $ \subseteq A$) s.t. $ \forall a \in A, \exists u \in U$ with $d_{\rho}(a, u) \leq \epsilon$. 
\item The \textbf{covering number} of $A$ is $N(\epsilon,A,\rho) := \inf_{U} |U|$ where $U$ is an $\epsilon$-cover for $A$.
\end{itemize}
\end{definition}
}%COLOR

We are given the set of $n$ instances $S := \{x_{i}\}_{i=1}^{n}$ with each $x_{i} \in \mathcal{X} \subseteq \mathbb{R}^{p}$ where $\mathcal{X} = \{x: \| x\| _r \leq X_{b}\}$, $2 \leq r \leq \infty$ and $X_{b}$ is a known constant. Let $\mu_{\mathcal{X}}$ be a probability measure on $\mathcal{X}$. Let $x_{i}$ be arranged as rows of a matrix $X$. We can represent the columns of $X= [x_{1}\;\hdots \; x_{n}]^{T}$ with $h_{j} \in \mathbb{R}^{n},j=1,...,p$, so $X$ can also be written as $[h_{1}{ } \cdots{ }h_{p} ]$.
Define function class $\mathcal{F}$ as the set of linear functionals whose coefficients lie in an $\ell_q$ ball and with a set of linear constraints:
\begin{align*}
\mathcal{F} &:= \{f: f(x) = \beta^{T}x, \beta \in \mathcal{B}\}\textrm{ where }\\
\mathcal{B}& := \left\{\beta \in \mathbb{R}^{p}: \| \beta\| _q \leq B_{b},
%\\ 
%&\hspace{25pt} 
\sum_{j=1}^{p}c_{j\nu}\beta_{j} +\delta_{\nu} \leq 1, \delta_{\nu} > 0, \nu=1,...,V\right\},
\end{align*}
where $1/r + 1/q = 1$ and $\{c_{j\nu}\}_{j,\nu}$, $\{\delta_{\nu}\}_{\nu}$ and $B_{b}$ are known constants. 
The linear constraints given by the $c_{j\nu}$'s force the hypothesis space $\mathcal{F}$ to be smaller, which will help with generalization - this will be shown formally by our main result in this section.
Let $\mathcal{F}_{|S}$ be defined as the restriction of $\mathcal{F}$ with respect to $S$.

Let $\{\tilde{c}_{j\nu}\}_{j,\nu}$ be proportional to $\{c_{j\nu}\}_{j,\nu}$:
\begin{eqnarray*}
\tilde{c}_{j\nu} &:=& \frac{c_{j\nu}n^{1/r}X_{b}B_{b}}{\| h_{j}\| _{r}} \;\;\textrm{ } \forall j=1,...,p \textrm{ and } \nu=1,...,V.\\
\end{eqnarray*}
Let $K$ be a positive number. Further, let the sets $P^{K}$ parameterized by $K$ and $P_{c}^{K}$ parameterized by $K$ and $\{\tilde{c}_{j\nu}\}_{j,\nu}$ be defined as
\begin{align}
P^{K} := \left\{(k_{1},...,k_{p}) \in \mathbb{Z}^{p}: \sum_{j=1}^{p}|k_{j}| \leq K\right\}.\nonumber \\
P_{c}^{K} := \left\{(k_{1},...,k_{p}) \in P^{K}: \sum_{j=1}^{p}\tilde{c}_{j\nu}k_{j} \leq K \; \forall \nu = 1,...,V\right\}.
\label{eqn:myconstraints1}
\end{align} 
\textcolor{Black}{ %COLOR FIRST LAYER
Let $|P^{K}|$ and $|P_{c}^{K}|$ be the sizes of the sets $P^{K}$ and $P_{c}^{K}$ respectively. The subscript $c$ in $P_{c}^{K}$ denotes that this polyhedron is a constrained version of $P^{K}$.
As the linear constraints given by the $c_{j\nu}$'s force the hypothesis space to be smaller, they force $|P_{c}^{K}|$ to be smaller.
Define $\matXbar$ to be equal to $X$ times a diagonal matrix whose $j^{th}$ diagonal element is $\scalej$.
Define $\lambda_{\min}(\matXbar^{T}\matXbar)$ to be the smallest eigenvalue of the matrix $\matXbar^{T}\matXbar$, which will thus be non-negative.
Using these definitions, we state our main result of this section.
}%COLOR
\begin{theorem}(Main result, covering number bound)
\begin{equation}
N(\sqrt{n}\epsilon,\mathcal{F}_{|S},\| \cdot\| _{2}) \leq 
\begin{cases}
\min\{|P^{K_{0}}|,|P_{c}^{K}|\} & \textrm{if } \epsilon < X_{b}B_{b} \\
1 & \textrm{ otherwise}
\end{cases},
\label{eqn:thm-statement}
\end{equation}
where 
\[
K_{0} = \ceil*{\frac{X_{b}^{2}B_{b}^{2}}{\epsilon^{2}}}
\] 
and
\[
K =\max\left\{K_{0}, \ceil*{\frac{nX_{b}^{2}B_{b}^{2}}{\lambda_{\min}(\matXbar^{T}\matXbar)\Big[\min_{\nu=1,...,V} \frac{\delta_{\nu}}{\sum_{j=1}^{p}|\tilde{c}_{j\nu}|}\Big]^{2}}}\right\}.
\] 

\label{thm:main-bound}
\end{theorem}
\textcolor{Black}{ %COLOR FIRST LAYER
The theorem gives a bound on the $\ell_{2}$ covering number for the specially constrained class $\mathcal{F}_{|S}$. The bound improves as the constraints given by $c_{j\nu}$ on the operational cost become tighter. In other words, as the $c_{j\nu}$ impose more restrictions on the hypothesis space, $|P_{c}^{K}|$ decreases, and the covering number bound becomes smaller. This bound can be plugged directly into an established generalization bound that incorporates covering numbers, and this is done in what follows to obtain Theorem \ref{thm:fullbound}.
}%COLOR
 
 Note that $\min\{|P^{K_{0}}|,|P_{c}^{K}|\}$ can be tighter than $|P_{c}^{K}|$ when $\epsilon$ is large. When $\epsilon$ is larger than $X_{b}B_{b}$, we only need one closed ball of radius $\sqrt{n}\epsilon$ to cover $\F_{|S}$, so $N(\sqrt{n}\epsilon,\mathcal{F}_{|S},\| \cdot\| _{2}) = 1$. In that case, the covering number in Theorem \ref{thm:main-bound} is appropriately bounded by $1$. If $\epsilon$ is large, but not larger than $X_{b}B_{b}$, then $|P_{c}^{K}|$ can be smaller than $|P^{K_{0}}|$. $|P^{K_{0}}|$ is the size of the polytope without the operational cost constraints. $|P_{c}^{K}|$ is the size of a potentially bigger polytope, but with additional constraints.

For this problem we generally assume that $n > p$; that is the number of examples is greater than the dimensionality $p$.  In such a case, $\lambda_{\min}(\matXbar^{T}\matXbar)$ can be shown to be bounded away from zero for a wide variety of distributions $\mu_{\mathcal{X}}$ (e.g., sub-gaussian zero-mean). When $\lambda_{\min}(\matXbar^{T}\matXbar) = 0$, the covering number bound becomes vacuous. 

%\subsection{Background on Generalization Bounds} %%%%%%%%%%%%%%%%%%%%%%%
Let us introduce some notation in order to state the generalization bound results. Given any function $f\in \mathcal{F}$, we would like to minimize the expected future loss (also known as the expected risk), defined as:
\begin{equation*}
R^{\textrm{true}}(l \circ f):=\mathbb{E}_{(x,y)\sim\mu_{\mathcal{X}\times \mathcal{Y}}}\Big[ l(f(x),y)\Big] = \int l(f(x),y) \partial \mu_{\mathcal{X}\times\mathcal{Y}} (x,y),
%\label{eqn:risk}
\end{equation*}
where $l: \mathcal{Y}\times\mathcal{Y}\rightarrow \R$ is the (fixed) loss function we had previously defined in Section \ref{sec:formulation}. The loss on the training sample (also known as the empirical risk) is:
\[
R^{\textrm{emp}}(l \circ f,\{(x_{i},y_{i})\}_1^n):= \frac{1}{n}\sum_{i=1}^n l(f (x_i),y_i).
\] 

We would like to know that $R^{\textrm{true}}(l \circ f)$ is not too much more than $R^{\textrm{emp}}(l \circ f,\{(x_{i},y_{i})\}_1^n)$, no matter which $f$ we choose from $\mathcal{F}$.
A typical form of generalization bound that holds with high probability for every function in $\mathcal{F}$ is
\begin{equation}
R^{\textrm{true}}(l\circ f) \leq R^{\textrm{emp}}(l \circ f, \{(x_{i},y_{i})\}_1^n) + \textrm{Bound(complexity($\mathcal{F}$),$n$)} \label{eqn:bound_nonprobForm},
\end{equation}
\textcolor{Black}{ %COLOR FIRST LAYER
where the complexity term takes into account the constraints on $\mathcal{F}$, both the linear constraints, and the $\ell_q$-ball constraint.
Theorem \ref{thm:main-bound} gives an upper bound on the term 
$\textrm{Bound(complexity($\mathcal{F}$),$n$)}$
 in (\ref{eqn:bound_nonprobForm}) above. In order to show this explicitly, we will give the definition of Rademacher complexity, restate how it appears in the relation between expected future loss and loss on training examples, and state an upper-bound for it in terms of the covering number.
}%COLOR

\begin{definition}
(Rademacher Complexity) The empirical Rademacher complexity of $\mathcal{F}_{|S}$ is\footnote{The factor $2$ in the defining equation (\ref{eqn:emp-rademacher-def}) is not very important. Some authors omit this factor and include it explicitly as a pre-factor in, for example, Theorem \ref{thm:complexity-rademacher-base}. }
\begin{equation}
\mathcal{\hat{R}}(\mathcal{F}_{|S}) = \mathbb{E}_{\sigma}\left[\sup_{f \in \mathcal{F}} \frac{2}{n}\sum_{i=1}^{n}\sigma_{i}f(x_{i})\right]\label{eqn:emp-rademacher-def}
\end{equation}
where $\{\sigma_i\}$ are Rademacher random variables ($\sigma_i = 1$ with prob$.$ $1/2$ and $-1$ with prob$.$ $1/2$). The Rademacher complexity is its expectation: $\mathcal{R}_n(\mathcal{F}) = \mathbb{E}_{S\sim (\mu_{\mathcal{X}})^{n}}[\mathcal{\hat{R}}(\mathcal{F}_{|S})]$.
\end{definition}
\textcolor{Black}{ %COLOR FIRST LAYER
The empirical Rademacher complexity $\mathcal{\hat{R}}(\mathcal{F}_{|S})$ can be computed given $S$ and $\mathcal{F}$, and by concentration, will be close to the Rademacher complexity. % \citep[formally shown by][]{bartlett02}.
%Define a function class $l_{\mathcal{F}}$, derived from the function class $\mathcal{F}$ by taking $l:\mathcal{Y}\times\mathcal{Y} \rightarrow \mathbb{R}$ and composing it with each element $f\in\mathcal{F}$. 
The following result relates the true risk to the empirical risk and empirical Rademacher complexity for any function class $\mathcal{H}$ 
%(see \citep{bartlett02} and references therein).
\citep[see][and references therein]{bartlett02}. 
Let the quantities $\mathcal{H}_{|S}, R^{true}(l\circ h)$ and $R^{\textrm{emp}}(l\circ h,\{x_{i},y_{i}\}_{1}^{n})$ be analogous to those we had defined for our specific class $\mathcal{F}$.
}%COLOR

\begin{theorem} (Rademacher Generalization Bound)
For all $\delta > 0$, with probability at least $1-\delta, \forall h \in \mathcal{H}$,
\begin{equation}
R^{\textrm{\rm true}}(l\circ h) \leq R^{\textrm{\rm emp}}(l\circ h,\{x_{i},y_{i}\}_{1}^{n}) + \mathcal{L}\cdot \mathcal{\hat{R}}(\mathcal{H}_{|S}) + \frac{3}{\sqrt{2}}\sqrt{\frac{\log\frac{1}{\delta}}{n}},
\label{eqn:bound_nonprobForm-rad}
\end{equation}
\label{thm:complexity-rademacher-base}
where $\mathcal{L}$ is the Lipschitz constant of the loss function.
\end{theorem}

Note that (\ref{eqn:bound_nonprobForm-rad}) is an explicit form of (\ref{eqn:bound_nonprobForm}). We will now relate $\mathcal{\hat{R}}(\mathcal{F}_{|S})$ to covering numbers thus justifying the importance of statement (\ref{eqn:thm-statement}) in Theorem \ref{thm:main-bound}. 
\textcolor{Black}{ %COLOR FIRST LAYER
In particular the following infinite chaining argument also known as Dudley's integral \citep[see][]{talagrand2005generic} relates $\hat{\mathcal{R}}(\mathcal{F}_{|S})$ to the covering number of the set $\mathcal{F}_{|S}$.
}%COLOR
\textcolor{Black}{ %COLOR FIRST LAYER
\begin{theorem}(Relating Rademacher Complexity to Covering Numbers) We are given that $\forall x\in\mathcal{X}$, we have $ f(x) \in [-X_{b}B_{b},X_{b}B_{b}]$. Then,
\begin{eqnarray*}
\frac{1}{X_{b}B_{b}}\mathcal{\hat{R}}(\mathcal{F}_{|S}) \leq 12 \int_{0}^{\infty}\sqrt{\frac{2 \log N (\alpha, \mathcal{F},L_{2}(\mu_{\mathcal{X}}^{n}))}{n}} d\alpha = 12 \int_{0}^{\infty}\sqrt{\frac{2 \log N (\sqrt{n}\alpha, \mathcal{F}_{|S},\| \cdot\| _{2})}{n}} d\alpha.
%\label{eqn:discretization-thm}
\end{eqnarray*}
%where $N(\alpha, \mathcal{F}_{|S},\| \cdot\| _{2})$ is the covering number of the set $\mathcal{F}_{|S}$.
\label{theorem:discretization}
\end{theorem}
}%COLOR

%\begin{theorem}(Relating Rademacher Complexity to Covering Numbers) Let $\forall x\in\mathcal{X}$, $f(x) \in$ $[-X_{b}B_{b}$,$X_{b}B_{b}]$. Then
%\begin{eqnarray*}
%\frac{1}{X_{b}B_{b}}\mathcal{\hat{R}}(\mathcal{F}_{|S}) \leq \inf_{\alpha}\Bigg( \sqrt{\frac{2 \log N (\alpha, \mathcal{F}_{|S},\| \cdot\| _{2})}{n}} + \alpha \Bigg)
%%\label{eqn:discretization-thm}
%\end{eqnarray*}
%where $N(\alpha, \mathcal{F}_{|S},\| \cdot\| _{2})$ is the covering number of the set $\mathcal{F}_{|S}$.
%\label{theorem:discretization}
%\end{theorem}
\textcolor{Black}{ %COLOR FIRST LAYER
Our main result in Theorem \ref{thm:main-bound} can be used in conjunction with Theorems \ref{thm:complexity-rademacher-base} and \ref{theorem:discretization}, to directly see how the true error relates to the empirical error and the constraints on the restricted function class $\mathcal{F}$ (the $\ell_{q}$-norm bound on $\beta$ and linear constraint on $\beta$ from the operational cost bias). Explicitly, that bound is here.
}%COLOR

\begin{theorem} (Generalization Bound for ML with Operational Costs)
\label{thm:fullbound}
For all $\delta > 0$, with probability at least $1-\delta, \forall f \in \mathcal{F}$, 
\begin{equation*}
R^{\textrm{\rm true}}(l\circ f) \leq R^{\textrm{\rm emp}}(l\circ f,\{x_{i},y_{i}\}_{1}^{n}) +12 \mathcal{L}{X_{b}B_{b}} 
 \int_{0}^{\infty}\sqrt{\frac{2 \log N(\sqrt{n}\epsilon,\mathcal{F}_{|S},\| \cdot\| _{2})}{n}} d\epsilon
+ \frac{3}{\sqrt{2}}\sqrt{\frac{\log\frac{1}{\delta}}{n}},
\end{equation*}
where 
\begin{align*}
N(\sqrt{n}\epsilon,\mathcal{F}_{|S},\| \cdot\| _{2}) \leq 
\begin{cases}
\min\{|P^{K_{0}}|,|P_{c}^{K}|\} & \textrm{if } \epsilon < X_{b}B_{b} \\
1 & \textrm{ otherwise}
\end{cases},
\end{align*} 
\[
K_{0} = \ceil*{\frac{X_{b}^{2}B_{b}^{2}}{\epsilon^{2}}},
\] 
and
\[
K = \max\left\{K_{0}, \ceil*{\frac{nX_{b}^{2}B_{b}^{2}}{\lambda_{\min}(\matXbar^{T}\matXbar)\Big[\min_{\nu=1,...,V} \frac{\delta_{\nu}}{\sum_{j=1}^{p}|\tilde{c}_{j\nu}|}\Big]^{2}}}\right\}
\] 
 are functions of $\epsilon$.
\end{theorem}
This bound implies that prior knowledge about the operational cost can be important for generalization. As our prior knowledge on the cost becomes stronger, the size of the hypothesis space becomes more restrictive, as seen through the constraints given by the $c_{j\nu}$. When this happens, the $|P_{c}^{K}|$ terms become smaller, and the whole bound becomes smaller. Note that the integral over $\epsilon$ is taken from $\epsilon=0$ to $\epsilon=\infty$. When $\epsilon$ is larger than $X_{b}B_{b}$, as noted earlier, $N(\sqrt{n}\epsilon,\mathcal{F}_{|S},\| \cdot\| _{2}) = 1$ and thus $\log N(\sqrt{n}\epsilon,\mathcal{F}_{|S},\| \cdot\| _{2}) = 0$.

Before we move onto building the necessary tools to prove Theorem \ref{thm:main-bound}, we compare our result with the bound in our work on the ML\&TRP \citep{TuRuJaadt11}. 
%Removed citation: ,TuRuJa11ArXiv
In that work, we considered a linear function class with a constraint on the $\ell_{2}$-norm and one additional linear inequality constraint on $\beta$. We then used a sample independent volumetric cap argument to get a covering number bound. Theorem \ref{thm:main-bound} is in some ways an improvement of the other result: (1) we can now have multiple linear constraints on $\beta$; (2) our new result involves a sample-specific bounding technique for covering numbers, which is generally tighter; (3) our result applies to $\ell_{q}$ balls for 
%\footnote{Restricting $r\geq 2$ restricts $q \in [1,2]$. In corollary \ref{corr:main-bound-ext} $q\geq1$, we show a way to get looser bounds for $r\geq 1$ so that we can have $q \in [1,\infty]$.} 
$q \in [1,2]$ whereas the previous analysis holds only for $q=2$. The volumetric argument in \citep{TuRuJaadt11}
%RZemoved citation: ,TuRuJa11ArXiv
 provided a scaling of the covering number. Specifically, the operational cost term for the ML\&TRP allowed us to reduce the covering number term in the bound from  $\sqrt{\log N(\cdot,\cdot,\cdot)}$ to $\sqrt{\log (\alpha N(\cdot,\cdot,\| \cdot\| _{2}))}$, or equivalently $\sqrt{\log N(\cdot,\cdot,\| \cdot\| _{2}) + \log \alpha}$, where $\alpha$ is a function of the operational cost constraint.  If $\alpha$ obeys $\alpha \ll 1$, then there is a noticeable effect on the generalization bound, compared to almost no effect when $\alpha \approx 1$. In the present work, the bound does not scale the covering number like this, instead it is a very different approach giving a more direct bound. 

\subsection{Proof of Theorem \ref{thm:main-bound}}
We make use of Maurey's Lemma \citep{barron93} in our proof \citep[in the same spirit as][]{zhang02}. The main ideas of Maurey's Lemma are used in many machine learning papers in various contexts 
%(\textit{e.g.}, \citep{koltchinskii2005complexities,schapire1998boosting,RudinSc09}). 
\citep[\textit{e.g.},][]{koltchinskii2005complexities,schapire1998boosting,RudinSc09}. Our proof of Theorem \ref{thm:main-bound} adapts Maurey's Lemma to handle polyhedrons, and allows us to apply counting techniques to bound the covering number. 
%\begin{proposition}\citep[Theorem 3 of][]{zhang02} If $\| x\| _r \leq b$ and $\| \beta\| _q \leq a$ where $1/r + 1/q = 1$ and $2 \leq r \leq \infty$, then
%\[
%\sup_{S\sim(\mu_{\mathcal{X}})^{n}} N(\epsilon,\mathcal{F}_{|S},\| \cdot\| _2) \leq (2p+1)^{\lceil \frac{a^{2}b^{2}}{\epsilon^2}\rceil}.
%\]
%\label{lemma:zhang}
%\end{proposition}

Recall that $X = [x_{1} \hdots x_{n}]^{T}$ was also defined column-wise as $[h_{1}\;\hdots\; h_{p}]$. 
%We will make a mild assumption on the matrix $X$ which is that the rank of $X$ is $p$. This means that the set of column vectors $\{h_{j}\}_{j=1}^{p}$ is linearly independent. This assumption will be useful in Lemma \ref{lemma:constraints}.
We introduce two scaled sets $\{\hS_{j}\}_{j}$ and $\{\betaS_{j}\}_{j}$ corresponding to $\{h_{j}\}_{j}$ and $\{\beta_{j}\}_{j}$ as follows:
\begin{align*}
\hS_{j} &:= \scalej h_{j} \; \textrm{ for } j =1,...,p; \textrm{ and}\\
\betaS_{j} &:= \invscalej \beta_{j} \; \textrm{ for } j =1,...,p.
\end{align*}
These scaled sets will be convenient in places where we do not want to carry the scaling terms separately. 

Any vector $y$ that is equal to $X\beta$ can thus be written in three different ways:
\begin{align*}
 y &= \sum_{j=1}^{p}\beta_{j}h_{j}, \textrm{ or}\\
 y &= \sum_{j=1}^{p}\betaS_{j}\hS_{j}, \textrm{ or}\\
 y &= \sum_{j=1}^{p}|\betaS_{j}|\sign(\betaS_{j})\hS_{j}.
\end{align*}

Our first lemma is a restatement of Maurey's lemma \citep[revised version of Lemma 1 in][]{zhang02}. We provide a proof based on the law of large numbers \citep{barron93} though other proof techniques also exist \citep[see][for a proof based on iterative approximation]{jones92}.
 
\textcolor{Black}{ %COLOR FIRST LAYER
The lemma states that every point $y$ in the convex hull of $\{h_j\}_j$ is close to one of the points $y_K$ in a particular finite set. 
%From this, we can infer a covering for the convex hull of $\{h_j\}_j$, which we do in Corollary \ref{corollary:maurey}.
}%COLOR

\begin{lemma}
Let $\max_{j=1,...,p} \| \hS_{j}\|$ be less than or equal to some constant  $b$. If $y$ belongs to the convex hull of set $\{\hS_{j}\}_{j}$, then for every positive integer $K \geq 1$, there exists $y_{K}$ in the convex hull of $K$ points of set $\{\hS_{j}\}_{j}$ such that $\| y - y_{K}\| ^{2} \leq \frac{b^2}{K}$.
\label{lemma:maurey}
\end{lemma}
\begin{proof} Let $y$ be written in the form:
$$y = \sum_{i=1}^{p}\gammaS_{j}\hS_{j},$$
where for each $j=1,...,p, \;\; \gammaS_{j} \geq 0$ and $\sum_{j=1}^{p}\gammaS_{j} \leq 1$. Let $\gammaS_{p+1} := 1- \sum_{j=1}^{p}\gammaS_{j} $. 

Consider a discrete distribution $\mathcal{D}$ formed by the coefficient vector $(\gammaS_1,..,\gammaS_p,\gammaS_{p+1})$. Associate a random variable $\mathsf{\hS}$ with support set $\{\hS_{1},...,\hS_{p},\mathbf{0}\}$. That is, $\textrm{Pr}(\mathsf{\hS}$ $=$ $\hS_{j}) = \gammaS_j$, $j=1,...,p$ and $\textrm{Pr}(\mathsf{\hS}=\mathbf{0}) = \gammaS_{p+1}$. 

Draw $K$ observations $\{\hS^{1},...,\hS^{K}\}$ uniformly and independently from $\mathcal{D}$ and form the sample average $\mathsf{y_{K}} := \frac{1}{K}\sum_{s=1}^{K} \hS^{s}$. Here, we are using the superscript index to denote the observation number.
The mean of this random variable $\mathsf{y}_{K}$ is: 
\begin{eqnarray*} 
\mathbb{E}_{\mathcal{D}}[\mathsf{y}_{K}] &=&\frac{1}{K}\sum_{s=1}^{K} \mathbb{E}_{\mathcal{D}}[\hS^{s}] \;\;\textrm{where}\\ 
\mathbb{E}_{\mathcal{D}}[\hS^{s}]&=&\sum_{j=1}^{p+1} \textrm{Pr}(\mathsf{\hS} = \hS_{j}) \hS_{j}= \sum_{j=1}^{p}\gammaS_{j}\hS_{j}  = y
\end{eqnarray*} 
hence $\mathbb{E}_{\mathcal{D}}[\mathsf{y}_{K}]= y$. 

The expected distance between $\mathsf{y}_{K}$ and $y$ is:
\begin{eqnarray}
\mathbb{E}_{\mathcal{D}}[\| \mathsf{y}_{K}-y\| ^{2}] &=& \mathbb{E}_{\mathcal{D}}[\| \mathsf{y}_{K} - \mathbb{E}_{\mathcal{D}}[\mathsf{y}_{K}]\| ^2]\nonumber
= \mathbb{E}\left[\sum_{i=1}^{n}(\mathsf{y}_{K} -  \mathbb{E}_{\mathcal{D}}[\mathsf{y}_{K}])_{i}^{2}\right]\nonumber\\
&\overset{(\dagger)}{=}& \sum_{i=1}^{n}\textrm{Var}((\mathsf{y}_{K})_{i})\nonumber
\overset{(*)}{=} \sum_{i=1}^{n}\frac{1}{K}\textrm{Var}((\mathsf{\hS})_{i})\nonumber\\
&\overset{(\ddagger)}{=}& \frac{1}{K}\sum_{i=1}^{n}\Big(\mathbb{E}_{\mathcal{D}}[(\mathsf{\hS})_{i}^{2}] - \mathbb{E}_{\mathcal{D}}[(\mathsf{\hS})_{i}]^{2}\Big) \nonumber
\overset{(\circ)}{=}\frac{1}{K}\Big(\mathbb{E}_{\mathcal{D}}[\| \mathsf{\hS}\| ^2] - \| \mathbb{E}_{\mathcal{D}}[\mathsf{\hS}]\| ^2 \Big) \nonumber\\
 &\leq&  \frac{1}{K}\mathbb{E}_{\mathcal{D}}[\| \mathsf{\hS}\| ^2] 
 \leq \frac{b^{2}}{K} 
\label{eqn:maurey_inequality2}
\end{eqnarray}
where we have used $i$ to be the index for the $i^{th}$ coordinate of the $n$ dimensional vectors. $(\dagger)$ follows from the definition of variance coordinate-wise. $(*)$ follows because each component of $\mathsf{y}_K$ is a sample average. $(\ddagger)$ also follows from the definition of variance. At step $(\circ)$, we rewrite the previous summations involving squares into ones that use the Hilbert norm. Our assumption  on $\max_{j=1,...,p}\|\hS_{j}\|$ tells us that $\mathbb{E}_{\mathcal{D}}[\| \mathsf{\hS}\| ^2] \leq b^{2}$ leading to (\ref{eqn:maurey_inequality2}). Since the squared Hilbert norm of the sample mean is bounded in this way, there exists a $y_{K}$ that satisfies the inequality, so that 
\[ \| y_{K} - y\| ^{2} \leq \frac{b^2}{K}. \] %\qed\]
\end{proof}

\textcolor{Black}{ %COLOR FIRST LAYER
The following corollary states explicitly that an approximation to $y$ exists that is a linear combination with coefficients chosen from a particular discrete set.
}%COLOR
\begin{corollary} For any $y$ and $K$ as considered above, we can find non-negative integers $m_{1},...,m_{p}$ such that $\sum_{j=1}^{p}m_{j} \leq K$ and $\|y - \sum_{j=1}^{p}\frac{m_{j}}{K}\hS_{j}\|^{2} \leq \frac{b^{2}}{K}$.
\label{corollary:maurey}
\end{corollary}
This follows immediately from the proof of Lemma \ref{lemma:maurey}, choosing $m_j$ to be the coefficients of the $\hS_j$'s such that $y_K=\sum_j \frac{m_j}{K}\hS_j$.

\textcolor{Black}{ %COLOR FIRST LAYER
The above corollary means that counting the number of $p$-tuple non-negative integers $m_{1},...,m_{p}$ gives us a covering of the set that $y$ belongs to. In the case of Lemma \ref{lemma:maurey}, this set is the convex hull of $\{\hS_{j}\}_{j}$. 
}%COLOR

\textcolor{Black}{ %COLOR FIRST LAYER
Before we can go further, we need to generalize the argument from the positive orthant of the $\ell_1$ ball to handle any coefficients that are in the whole unit-length $\ell_{1}$-ball. This is what the following lemma accomplishes.
}%COLOR
\begin{lemma}
Let $\max_{j=1,...,p} \| \hS_{j}\|$ be less than or equal to some constant  $b$. For any $y = \sum_{j=1}^{p}\betaS_{j}\hS_{j}$ such that $\|\betaS\|_{1} \leq 1$, given a positive integer $K$, we can find a $y_{K}$ such that
\begin{align*}
\|y - y_{K}\|_{2}^{2} \leq \frac{b^{2}}{K}
\end{align*}
where $y_{K} = \sum_{j=1}^{p}\frac{k_{j}}{K}\hS_{j}$ is a combination of $\{\hS_{j}\}$ with integers $k_{1},...,k_{p}$ such that $\sum_{j=1}^{p}|k_{j}| \leq K$.
\label{lemma:maurey-ext}
\end{lemma}
\begin{proof}
Lemma \ref{lemma:maurey} cannot be applied directly since the $\{\betaS_{j}\}_{j}$  can be negative. We rewrite $y$ or equivalently $\sum_{j=1}^{p}\betaS_{j}\hS_{j}$ as
\begin{align*}
y=\sum_{j=1}^{p}|\betaS_{j}|\sign(\betaS_{j})\hS_{j}.
\end{align*}
Thus $y$ lies in the convex combination of $\{\sign(\betaS_{j})\hS_{j}\}_{j}$. Note that this step makes the convex hull depend on the $y$ or $\{\betaS_{j}\}_{j}$ we start with. Nonetheless, we know by substituting $\{\sign(\betaS_j)\hS_j\}_j$ for $\{\hS_j\}_j$ in the statement of Lemma \ref{lemma:maurey} and Corollary \ref{corollary:maurey} that  
\begin{enumerate}
\item we can find $y_{K}$, or equivalently
\item we can find non-negative integers $m_{1},...,m_{p}$ with $\sum_{j=1}^{p}m_{j} \leq K$,
\end{enumerate}
such that $\|y - y_{K}\|_{2}^{2} \leq \frac{b^{2}}{K}$ where $y_{K} = \sum_{j=1}^{p}\frac{m_{j}}{K}\sign(\betaS_{j})\hS_{j}$ holds. This implies there exist integers $k_{1},...,k_{p}$ such that $y_{K} = \sum_{j=1}^{p}\frac{k_{j}}{K}\hS_{j}$ where $\sum_{j=1}^{p}|k_{j}| \leq K$. We simply let $k_{j} = m_{j}\sign(\betaS_{j})$. Thus, we absorbed the signs of the $\betaS_j$'s, and the coefficients no longer need to be nonnegative. 

In other words, we have shown that if a particular $y_{K}$ is in the convex hull of points $\{\sign(\betaS_{j})\hS_{j}\}_{j}$, then the same $y_{K}$ is a linear combination of $\{\hS_{j}\}_{j}$ where the coefficients of the combination $k_{1}/K,...,k_{p}/K$ obey $\sum_{j=1}^{p}|k_{j}| \leq K$. This concludes the proof. %\qed
\end{proof}

\textcolor{Black}{ %COLOR FIRST LAYER
We now want to answer the question of whether the $k_{1}/K,...,k_{p}/K$ can obey (related) linear constraints if the original $\{\betaS_{j}\}_{j}$ did so. These constraints on the $\{\betaS_{j}\}_{j}$'s are the ones coming from constraints on the operational cost. 
In other words, we want to know that our (discretized) approximation of $y$ also obeys a constraint coming from the operational cost.
}%COLOR

\textcolor{Black}{ %COLOR FIRST LAYER
Let $\{\betaS_{j}\}_{j}$ satisfy the linear constraints within the definition of $\mathcal{B}$, in addition to satisfying $\|\betaS\|_{1} \leq 1$: 
\begin{equation*}%\label{epsiconstraints}
\sum_{j=1}^{p}\tilde{c}_{j\nu}\betaS_{j} + \delta_{\nu} \leq 1, \textrm{ for fixed }\delta_{\nu} > 0,  \nu=1,...,V. 
\end{equation*} 
We now want that for large enough $K$, the $p$-tuple $k_{1}/K,...,k_{p}/K$ also meets certain related linear constraints. %That is, we want the discretized version of $y$ to also obey an operational cost constraint.
}%COLOR

We will make use of the matrix $\matXbar$, defined before Theorem \ref{thm:main-bound}. It has the elements of the scaled set $\{\hS_{j}\}_{j}$ as its columns: $\matXbar := [\hS_{1} \; \hdots \; \hS_{p}]$.

\begin{lemma}
Take any $y = \sum_{j=1}^{p}\betaS_{j}\hS_{j}$, and any $y_{K} = \sum_{j=1}^{p}\frac{k_{j}}{K}\hS_{j}$, with: 
\begin{align*}
\sum_{j=1}^{p}\tilde{c}_{j\nu}\betaS_{j} + \delta_{\nu} &\leq 1,  \textrm{ for fixed }\delta_{\nu} > 0, \nu=1,...,V \textrm{ where } \|\betaS\|_{1} \leq 1
\end{align*} 
and $\|y-y_{K}\|_{2}^{2} \leq b^{2}/K$.
Whenever
\[
K \geq \frac{b^{2}}{\left[\min_{\nu=1,...,V}\frac{\delta_{\nu}}{\sum_{j=1}^{p}|\tilde{c}_{j\nu}|}\right]^{2} \lambda_{\min}(\matXbar^{T}\matXbar)},
\]
then the following linear constraints on $k_{1}/K,...,k_{p}/K$ hold:
\[\sum_{j=1}^{p}\tilde{c}_{j\nu}\frac{k_{j}}{K} \leq 1, \textrm{ } \nu=1,...,V.\] 
\label{lemma:constraints}
\end{lemma}
\textcolor{Black}{ %COLOR FIRST LAYER
This lemma states that as long as the discretization is fine enough, our approximation $y_K$ obeys similar operational cost constraints to $y$.
}%COLOR

\begin{proof}

Let $\kappa := [k_{1}/K\; \hdots \; k_{p}/K]^{T}$.
%We know that the columns of $X = [h_{1},\; \hdots \; h_{p}]$ are linearly independent and rank$(X) = p$. This means that the $p\times p$ matrix $X^{T}X$ is positive definite.
Using the definition of $\matXbar$,
\begin{align}\nonumber
\frac{b^{2}}{K} &\geq \|y - y_{K}\|_{2}^{2} = \|\matXbar\betaS - \matXbar\kappa\|_{2}^{2} = \|\matXbar(\betaS - \kappa)\|_{2}^{2}\\\label{eqn:lambdamin}
& = (\betaS - \kappa)^{T}\matXbar^{T}\matXbar(\betaS - \kappa) \overset{(*)}{\geq} \lambda_{\min}(\matXbar^{T}\matXbar)\|\betaS - \kappa\|_{2}^{2}.
\end{align}
In $(*)$, we used the fact that for a positive (semi-)definite matrix $M$ and for every non-zero vector $z$, $z^{T}Mz \geq \lambda_{\min}(M)z^{T}Iz$. (If $\betaS = \kappa$, we are done since $\kappa$ will obey the constraints $\betaS$ obeys.)
Also, for any $z$, in each coordinate $j$, $|z_{j}| \leq \max_{j=1,...,p}|z_{j}| = \|z\|_{\infty} \leq \|z\|_{2}$. Combining this with (\ref{eqn:lambdamin}), we have:
\begin{align*}
\left|\betaS_{j} - \frac{k_{j}}{K}\right| \leq \|\betaS - \kappa\|_{2} \leq \frac{b}{\sqrt{K\lambda_{\min}(\matXbar^{T}\matXbar)}}.
\end{align*}

This implies that $\kappa$ itself component-wise satisfies
\[
\betaS_{j} -A \leq \frac{k_{j}}{K} \leq \betaS_{j} +A \textrm{ where } A := \frac{b}{\sqrt{K \lambda_{\min}(\matXbar^{T}\matXbar)}}.\]

%\begin{align*}
%\betaS_{j} -\frac{b}{\sqrt{K \lambda_{\min}(\matXbar^{T}\matXbar)}} \leq \frac{k_{j}}{K} \leq \betaS_{j} + \frac{b}{\sqrt{K\lambda_{\min}(\matXbar^{T}\matXbar)}}.
%\end{align*}
%Let $A := \frac{b}{\sqrt{K \lambda_{\min}(\matXbar^{T}\matXbar)}}$.

So far we know that for all $\nu=1,...,V$, $\sum_{j=1}^{p}\tilde{c}_{j\nu}\betaS_{j} + \delta_{\nu} \leq 1$, with $\delta_{\nu} > 0$,  and each coordinate $k_{j}/K$ within $\kappa$ varies from $\betaS_{j}$ by at most an amount $A$. We would like to establish that the linear constraints $\sum_{j=1}^{p}\tilde{c}_{j\nu}\frac{k_{j}}{K} \leq 1, \textrm{ } \nu=1,...,V;$ always hold for such a $\kappa$.
For each constraint $\nu$, substituting the extremal values of $k_{j}$ according to the sign of $\tilde{c}_{j\nu}$, we get the following upper bound:
\begin{align*}
\sum_{j=1}^{p}\tilde{c}_{j\nu}\frac{k_{j}}{K}
\leq & \sum_{\tilde{c}_{j\nu}>0}\tilde{c}_{j\nu}(\betaS_{j} + A) + \sum_{\tilde{c}_{j\nu} < 0} \tilde{c}_{j\nu}(\betaS_{j} - A)
 = \sum_{j=1}^{p}\tilde{c}_{j\nu}\betaS_{j} + A\sum_{j=1}^{p}|\tilde{c}_{j\nu}|.
\end{align*}

This sum $ \sum_{j=1}^{p}\tilde{c}_{j\nu}\betaS_{j} + A\sum_{j=1}^{p}|\tilde{c}_{j\nu}|$ is less than or equal to $1$ iff $A\sum_{j=1}^{p}|\tilde{c}_{j\nu}| \leq \delta_{\nu}$.

Thus we would like $A \leq \frac{\delta_{\nu}}{\sum_{j=1}^{p}|\tilde{c}_{j\nu}|}$ for all $\nu=1,...,V$. That is,
\begin{align*}
{}\;\;\; &\frac{b}{\sqrt{K \lambda_{\min}(\matXbar^{T}\matXbar)}}= A \leq \min_{\nu=1,...,V}\frac{\delta_{\nu}}{\sum_{j=1}^{p}|\tilde{c}_{j\nu}|}\\
\Leftrightarrow\;\;\; & K \geq \frac{b^{2}}{\left[\min_{\nu=1,...,V}\frac{\delta_{\nu}}{\sum_{j=1}^{p}|\tilde{c}_{j\nu}|}\right]^{2} \lambda_{\min}(\matXbar^{T}\matXbar)}. %\qed
\end{align*}
\end{proof}

\textcolor{Black}{ %COLOR FIRST LAYER
We now proceed with the proof of our main result of this section. The result involves covering numbers, where the cover for the set will be the vectors with discretized coefficients that we have been working with in the lemmas above.\\
}%COLOR

\begin{proof} \textit{(of Theorem \ref{thm:main-bound})}

Recall that
\begin{itemize}
\item the matrix $X$ is defined as $[h_{1}\textrm{ } ... \textrm{ } h_{p}]$;
\item the scaled versions of vector $\{h_{j}\}_{j}$ are $\hS_{j} = \scalej h_{j}$ for $j=1,...,p$;
\item the scaled versions of coefficients $\{\beta_{j}\}_{j}$ are $\betaS_{j} = \invscalej \beta_{j}$ for $j=1,...,p$; and
\item any vector $y = X\beta = \sum_{j=1}^{p}\beta_{j}h_{j}$ can be rewritten as $\sum_{j=1}^{p}\betaS_{j}\hS_{j}$.
\end{itemize}
\textcolor{Black}{ %COLOR FIRST LAYER
We will prove three technical facts leading up to the result.\\
}%COLOR

\noindent\textbf{Fact 1.} If $\|\beta\|_{q} \leq B_{b}$, then $\|\betaS\|_{1} \leq 1$.\\

Because $1/r + 1/q = 1$,  by H{\"o}lder's inequality we have:
\begin{align}\label{holderseqn}
\sum_{j=1}^{p}|\betaS_{j}|  =& \frac{1}{n^{1/r}B_{b}X_{b}}\sum_{j=1}^{p} \| h_{j}\| _{r}|\beta_{j}|\\ 
\leq & \frac{1}{n^{1/r}B_{b}X_{b}}\left(\sum_{j=1}^{p} \| h_{j}\| _{r}^{r}\right)^{1/r}\left(\sum_{j=1}^{p} | \beta^{j}| ^{q}\right)^{1/q}. \nonumber
\end{align} 
To bound the above notice that in our notation, $(h_{j})_{i} = (x_{i})_{j}$. That is, the $i^{th}$ component of feature vector $h_{j}$, i.e., $(h_{j})_{i}$ is also the $j^\textrm{th}$ component of example $x_i$.  Thus,
\begin{align*}
\left(\sum_{j=1}^{p} \| h_{j}\| _{r}^{r}\right)^{1/r} =& \left(\sum_{j=1}^p\sum_{i=1}^n \left((h_{j})_{i}\right)^r\right)^{1/r} =\left(\sum_{i=1}^n\sum_{j=1}^p\left((h_{j})_{i}\right)^r\right)^{1/r}\\
=&\left(\sum_{i=1}^n\|x_i\|_r^r\right)^{1/r} \leq \left(nX_b^r\right)^{1/r} = n^{1/r}X_b.
\end{align*}
Plugging this into (\ref{holderseqn}), and using the fact that $\|\beta\|_{q} \leq B_{b}$, we have
\[
\sum_{j=1}^{p}|\betaS_{j}|  \leq \frac{1}{n^{1/r}B_{b}X_{b}} n^{1/r}X_{b}  B_{b} = 1,
\]
that is,  $\|\betaS\|_{1} \leq 1$.\\

\noindent \textbf{Fact 2.} Corresponding to the set of linear constraints on $\beta$:
\begin{align*}
\sum_{j=1}^{p}c_{j\nu}\beta_{j} +\delta_{\nu} \leq 1, \delta_{\nu} > 0, \nu=1,...,V,
\end{align*}
there is a set of linear constraints on $\betaS_{j}$, namely $\sum_{j=1}^{p}\tilde{c}_{j\nu} \betaS_{j} + \delta_{\nu} \leq 1, \nu=1,...,V$.\\

Recall that $\beta \in \mathcal{B}$ also means that $\sum_{j=1}^{p}c_{j\nu}\beta_{j} + \delta_{\nu} \leq 1$ for some $\delta_{\nu} > 0$ for all $\nu = 1,...,V$.
Thus, for all $\nu = 1,...,V$:
\begin{eqnarray*}
&& \sum_{j=1}^{p}c_{j\nu}\beta_{j} + \delta_{\nu} \leq 1 \\
&\Leftrightarrow & \sum_{j=1}^{p}c_{j\nu}\left(\scalej \invscalej\right)\beta_{j} + \delta_{\nu} \leq 1\\
&\Leftrightarrow & \sum_{j=1}^{p}\tilde{c}_{j\nu} \betaS_{j} + \delta_{\nu} \leq 1
\end{eqnarray*}
which is the set of corresponding linear constraints on $\{\betaS_{j}\}_{j}$ we want.\\

\noindent\textbf{Fact 3.} $\forall j=1,...,p $, $\; \| \hS_{j}\| _{2} \leq n^{1/2}X_{b}B_{b}$.\\

Jensen's inequality implies that for any vector $z$ in $\mathbb{R}^n$, and for any $r\geq 2$, it is true that $\frac{1}{n^{1/2}} \|z\|_2\leq \frac{1}{n^{1/r}}\|z\|_r$. 
Using this for our particular vector $\hS_{j}$ and our given $r$, 
we get 
\begin{align*}
 \| \hS_{j}\| _{2}\leq \| \hS_{j}\| _{r} n^{1/2}\frac{1}{n^{1/r}}.
\end{align*}
But we know
\begin{align*}
\| \hS_{j}\| _{r} = \left\|\scalej h_{j} \right\|_{r} = \scalej \|h_{j}\|_{r} = n^{1/r}X_{b}B_{b}.
\end{align*}
Thus, we have $\| \hS_{j}\| _{2} \leq n^{1/2}X_{b}B_{b}$ for each $j$, and thus,
$\max_{j=1,...,p}\| \hS_{j}\| _{2} \leq n^{1/2}X_{b}B_{b}$.\\

\textcolor{Black}{ %COLOR FIRST LAYER
With those three facts established, we can proceed with the proof of Theorem \ref{thm:main-bound}. 
}%COLOR
Facts 1 and 2 show that the requirements on $\betaS$ for Lemma \ref{lemma:maurey-ext} and Lemma \ref{lemma:constraints} are satisfied. Fact 3 shows that the requirement on $\{\hS_{j}\}_{j}$ for Lemma \ref{lemma:maurey-ext} is satisfied with constant $b$ being set to  $n^{1/2}X_{b}B_{b}$.
Since the requirements on $\{\hS_{j}\}_{j}$ and $\{\betaS_{j}\}_{j}$ are satisfied, we want to choose the right value of positive integer $K$ such that Lemma \ref{lemma:constraints} is satisfied and also we would like the squared distance between $y$ and $y_K$ to be less than $n\epsilon^2$. To do this, we pick $K$ to be the bigger of the two quantities: $X_{b}^{2}B_{b}^{2}/\epsilon^{2}$ and that given in Lemma \ref{lemma:constraints}. That is,
\begin{align}
K = \ceil*{\max\left\{\frac{X_{b}^{2}B_{b}^{2}}{\epsilon^{2}}, \frac{nX_{b}^{2}B_{b}^{2}}{\left[\min\limits_{\nu=1,...,V}\frac{\delta_{\nu}}{\sum_{j=1}^{p}|\tilde{c}_{j\nu}|}\right]^{2} \lambda_{\min}(\matXbar^{T}\matXbar)}\right\}}.
\label{eqn:chooseK}
\end{align}
\textcolor{Black}{ %COLOR FIRST LAYER
This will force our discretization for the cover to be sufficiently fine that things will work out: we will be able to count the number of cover points in our finite set, and that will be our covering number. 
}%COLOR

To summarize, with this choice, for any $y \in \F_{|S}$, we can find integers $k_{1},...,k_{p}$
such that the following hold simultaneously:
\begin{enumerate}
\item[$a$.] (It gives a valid discretization of $y$.) $\sum_{i=1}^{p}|k_{i}| \leq K$,
\item[$b$.] (It gives a good approximation to $y$.) The approximation $y_{K} = \sum_{j=1}^{p}\frac{k_{i}}{K}\hS_{j}$ is $\epsilon\sqrt{n}$ close to $y = \sum_{j=1}^{p}\betaS_{j}\hS_{j}$. That is,
\begin{align*}
\|y - y_{K}\|_{2}^{2} \leq \frac{nX_{b}^{2}B_{b}^{2}}{K} \leq n\epsilon^{2}, \textrm{and}
\end{align*}
\item[$c$.] (It obeys operational cost constraints.) $\sum_{j=1}^{p}\tilde{c}_{j\nu}\frac{k_{j}}{K} \leq 1, \textrm{ } \nu=1,...,V$.
\end{enumerate}
In the above, the existence of $k_{1},...,k_{p}$ satisfying $(a)$ and $(b)$ comes from Lemma \ref{lemma:maurey-ext} where we have also used $K$ satisfying $K \geq X_{b}^{2}B_{b}^{2}/\epsilon^{2} \geq 1$. Lemma \ref{lemma:constraints} along with the choice of $K$ from (\ref{eqn:chooseK}) guarantees that $(c)$ holds as well for this choice of $k_{1},...,k_{p}$.

Thus, by $(b)$, any $y \in \F_{|S}$ is within $\epsilon\sqrt{n}$ in $\ell_{2}$ distance of at least one of the vectors with coefficients $k_{1}/K,...,k_{p}/K$. Therefore counting the number of $p$-tuple integers $k_{1},...,k_{p}$ such that $(a)$ and $(c)$ hold, or equivalently the number of solutions to (\ref{eqn:myconstraints1}), gives a bound on the covering number, which is $|P_{c}^{K}|$. That is,
\begin{align*}
N(\sqrt{n}\epsilon,\mathcal{F}_{|S},\| \cdot\| _{2}) \leq |P_{c}^{K}|.
\end{align*}

If we did not have any linear constraints, we would have the following bound,
\begin{align*}
N(\sqrt{n}\epsilon,\mathcal{F}_{|S},\| \cdot\| _{2}) \leq |P^{K_{0}}|,
\end{align*}
where $K_{0}:=\ceil*{\frac{X_{b}^{2}B_{b}^{2}}{\epsilon^{2}}}$ by using Lemma \ref{lemma:maurey-ext} and very similar arguments as above.

In addition, when $\epsilon\geq X_{b}B_{b}$, the covering number is exactly equal to $1$ since we can cover the set $\F_{|S}$ by a closed ball of radius $\sqrt{n}X_{b}B_{b}$.

%We note that when $\epsilon$ is large enough, the quantity $K$ in (\ref{eqn:chooseK}) (which is max of a term involving $\epsilon$ and a term not involving $\epsilon$) becomes independent of $\epsilon$ and is only determined by the constraints $\{\tilde{c}_{j\nu}\}_{j,\nu}$. In this case, $|P^{K_{0}}| \rightarrow 1$ whereas $|P_{c}^{K}|$ may be equal to a value greater than $1$. This means that $|P^{K_{0}}|$ is a better upper bound for the covering number than $P_{c}^{K}$.

Thus we modify our upper bound by taking the minimum of the two quantities $|P^{K_{0}}|$ and $|P_{c}^{K}|$ appropriately to get the result:
\begin{align*}
N(\sqrt{n}\epsilon,\mathcal{F}_{|S},\| \cdot\| _{2}) \leq 
\begin{cases}
\min\{|P^{K_{0}}|,|P_{c}^{K}|\} & \textrm{if } \epsilon < X_{b}B_{b} \\
1 & \textrm{ otherwise.}
\end{cases}
\end{align*}
\end{proof}

%The restriction on $r\geq2$ can be relaxed at the expense of getting a much looser covering bound.
%\begin{corollary}(Extending Theorem \ref{thm:main-bound} for $r \in [1,2]$)
%For $1\leq r \leq 2$ which defines $X_{b}$, if
%\[
%K \geq \max\left\{\frac{n^{\frac{2}{r}-1}X_{b}^{2}B_{b}^{2}}{\epsilon^{2}}, \frac{n^{\frac{2}{r}}X_{b}^{2}B_{b}^{2}}{\lambda_{\min}(\matXbar^{T}\matXbar)\left[\min_{\nu=1,...,V} \frac{\delta_{\nu}}{\sum_{j=1}^{p}|\tilde{c}_{j\nu}|}\right]^{2}}\right\},
%\]  
%\begin{equation}
%\textrm{then }  N(\sqrt{n}\epsilon,\mathcal{F}_{|S},\| \cdot\| _{2}) \leq |P_{c}^{K}|.
%\label{eqn:thm-statement-ext}
%\end{equation}
%\label{corr:main-bound-ext}
%\end{corollary}
%\begin{proof}
%In Fact 3 of the proof of Theorem \ref{thm:main-bound}, we used Jensen's inequality and $r\geq 2$ to get,
%\begin{align*}
% \| \hS_{j}\| _{2}\leq \| \hS_{j}\| _{r} n^{1/2}\frac{1}{n^{1/r}}.
%\end{align*}
%For $r \in [1,2]$ the following also holds:
%\begin{align*}
% \| \hS_{j}\| _{2}\leq \| \hS_{j}\| _{r}.
%\end{align*}
%This implies that our constant $b = \max_{j=1,...,p}\| \hS_{j}\| _{2}$ is now less than $ n^{1/r}X_{b}B_{b}$. Substituting this new value of $b$ in the expression for $K$ gives us the above result.
%\end{proof}

Since Theorem \ref{thm:main-bound} suggests that $|P_{c}^{K}|$ may be an important quantity for the learning process, we discuss how to compute it. We assume that $\tilde{c}_{j\nu}$  are rationals for all $j=1,..,p, \nu=1,...,V,$ so that we can multiply each of the $V$ constraints describing $P_{c}^{K}$ by the corresponding gcd of the $p$ denominators. This is without loss of generality because the rationals are dense in the reals. This ensures that all the constraints describing polyhedron $P_{c}^{K}$ have integer coefficients. Once this is achieved, we can run Barvinok's algorithm \citep[using for example, Lattice Point Enumeration, see][and references therein]{deloera05} that counts integer points inside polyhedra and runs in polynomial time for fixed dimension (which is $p$ here). Using the output of this algorithm within our generalization bound will yield a much tighter bound than in previous works \citep[for example, the bound in][Theorem 3]{zhang02}, especially when $(r,q) = (\infty,1)$; this is true simply because we are counting more carefully. Note that counting integer points in polyhedrons is a fundamental question in a variety of fields including number theory, discrete optimization, combinatorics to name a few, and making an explicit connection to bounds on the covering number for linear function classes can potentially open doors for better sample complexity bounds.

%%%%%%%%%%%%%%%%%%%%%%%%%%%%Section:Related works
\section{Discussion and Conclusion}\label{sec:relatedworks}
\textcolor{Black}{ %COLOR FIRST LAYER
The perspective taken in this work contrasts with traditional decision analysis and predictive modeling; in these fields, a single decision is often the only end goal. Our goal involves exploring how predictive modeling influences decisions and their costs. Unlike traditional predictive modeling, our regularization terms involve optimization problems, and are not the usual vector norms.
}%COLOR

\textcolor{Black}{ %COLOR FIRST LAYER
The simultaneous process serves as a way to understand uncertainty in decision-making, and can be directly applied to real problems. 
\textcolor{Black}{ %COLOR SECOND LAYER
We centered our discussion and demonstrations around three questions, namely: ``What is a reasonable amount to allocate for this task so we can react best to whatever nature brings?" (answered in Section \ref{sec:experiments}),  ``Can we produce a reasonable probabilistic model, supported by data, where we might expect to pay a specific amount?" (answered in Section \ref{sec:experiments}), and ``Can our intuition about how much it will cost to solve a problem help us produce a better probabilistic model?" (answered in Section \ref{sec:bound}). The first two were answered by exploring how optimistic and pessimistic views can influence the probabilistic models and the operational cost range. Given the range of reasonable costs, we could allocate resources effectively for whatever nature brings. Also given a specific cost value, we could pick a corresponding probabilistic model and verify that it can be supported by data. The third question was comprehensively answered in Section \ref{sec:bound} by evaluating how intuition about the operational cost can restrict the probabilistic model space and in turn lead to better sample complexity if the intuition is correct.
}
These are questions that are not handled in a natural way by current paradigms. Answering these three questions are not the only uses for the simultaneous process. For instance, domain experts could use the simultaneous process to explore the space of probabilistic models and policies, and then simply pick the policy among these that most agrees with their intuition. Or, they could use the method to refine the probabilistic model, in order to exclude solutions that the simultaneous process found that did not agree with their intuition.
}%COLOR

\textcolor{Black}{ %COLOR FIRST LAYER
The simultaneous process is useful in cases where there are many potentially good probabilistic models, yielding a large number of (optimal-response) policies. This happens when the training data are scarce, or the dimensionality of the problem is large compared to the sample size, and the operational cost is not smooth. These conditions are not difficult to satisfy, and do occur commonly. For instance, data can be scarce (relative to the number of features) when they are expensive to collect, or when each each instance represents a real-world entity where few exist; for instance, each example might be a product, customer, purchase record, or historic event. Operational cost calculations commonly involve discrete optimization; there can be many scheduling, knapsack, routing, constraint-satisfaction, facility location, and matching problems, well beyond what we considered in our simple examples. The simultaneous process can be used in cases where the optimization problem is difficult enough that sampling the posterior of Bayesian models, with computing the policy at each round, is not feasible.
}%COLOR

\textcolor{Black}{ %COLOR FIRST LAYER
We end the paper by discussing the applicability of our policy-oriented estimation strategy in the real world. Prediction is the end goal for machine learning problems in vision, image processing and biology, and in other scientific domains, but there are many domains where the learning algorithm is used to make recommendations for a subsequent task. We showed applications in Section \ref{sec:experiments} but it is not hard to find applications in other domains, where using either the traditional sequential process, decision theory, or robust optimization may not suffice. Here are some other potential domains:
}%COLOR
\begin{itemize}

\item Internet advertising, where the goal of the advertising platform is to choose which ad to show a customer. For each customer and advertiser, there is an uncertain estimate of the probability that the customer will click the ad from that advertiser. These estimates determine which ad will be shown next, which is a discrete decision \citep{muthukrishnan07}.

\item Portfolio management, where we allocate our budget among $n$ risky assets with uncertain returns, and each asset has a different cost associated with the investment \citep{konno1991mean}.

\item Maintenance applications \citep[in addition to the ML\&TRP][]{TuRuJaadt11}, where we estimate probabilities of failure for each piece of equipment, and create a policy for repairing, inspecting, or replacing the equipment. Certain repairs are more expensive than others, so the costs of various policy decisions could potentially change steeply as the probability model changes.
% (see the discussion in \citep{TuRuJa11ArXiv}).

\item Traffic flows on transportation networks, where the problem can be that of load balancing based on resource constraints and forecasted demands \citep{KoulaEtAl}. 

\item Policy decisions based on dynamical system simulations, for instance, climate
policy, where a politician wants to understand the uncertainty in policy
decisions based on the results of a large-scale simulation. If the
simulation cannot be computed for all initial values, its result can be
estimated using a machine learning algorithm \citep{BartonNX10}.

\item Pharmaceutical companies choosing a subset of possible drug targets to test,
where the drugs are predicted to be effective, and cannot be overly expensive to
produce \citep{Yu2012}. This might be similar in many ways to the real-estate purchasing problem discussed in Section \ref{sec:experiments}.

\item Machine task scheduling on multi-core processors, where we need to allocate processors to various jobs during a large computation. This could be very similar to the problem of scheduling with constraints addressed in Section \ref{sec:experiments}. If we optimistically estimate the amount of time each job takes, we will hopefully free up processors on time so they can be ready for the next part of the computation.

\end{itemize}

\textcolor{Black}{ %COLOR FIRST LAYER
We believe the simultaneous process will open the door for other methods dealing with the interaction of machine learning and decision-making that fall outside the realm of the usual paradigms.
}%COLOR

%Because we allow the regularization term in the simultaneous process to be the solution of an optimization problem,  it is possible, by considering a wide range of optimization problems, to capture a very broad variety of situations.

%%%%%%%%%%%%%%%%%%Section:Conclusion
%\section{Conclusion}\label{sec:conclusion}

%The perspective taken in this work contrasts with classical (and non-classical) statistics and machine learning; in those fields, prediction is often the only end goal. Our goal involves also how the model is used. There are many possible scenarios where including operational costs in statistical modeling could be very useful. In particular, this occurs when data are scarce or noisy, when the dimensionality is large and there is a lot of uncertainty in the model predictions, and when the operational cost has a steep gradient near the minimizer of the regularized loss. In our work, regularization terms involve optimization problems, not simply vector norms. We presented several example applications where including the operational cost substantially influenced the quality of the solution. Constraints on the operational costs lead to new types of hypothesis spaces, and we have obtained generalization bounds for a new type of hypothesis space involving arbitrary linear constraints.

\section*{Acknowledgements}
Funding for this project comes in part from a Fulbright Science and Technology Fellowship, an award from the Solomon Buchsbaum Research Fund, and NSF grant IIS-1053407.

\bibliography{Bib_Theja}
\end{document}